\pdfoutput=1

\documentclass{article}[11pt]
\usepackage{fullpage}

\usepackage[utf8]{inputenc}
\usepackage{hyperref}
\usepackage{amsfonts}
\usepackage{amsmath}
\usepackage{amssymb}
\usepackage{amsthm}
\usepackage{qtree}
\usepackage{graphicx}
\usepackage{booktabs}
\usepackage{subcaption}
\usepackage[section]{placeins}
\graphicspath{ {./figures/} }
\usepackage{xcolor}
\usepackage[english]{babel}
\usepackage[ruled,vlined]{algorithm2e}
\usepackage{bbm}

\usepackage{url}
\usepackage{float}
\usepackage{enumitem}
\usepackage{bbm}
\usepackage{mathtools}
\usepackage{multirow}

\definecolor{DarkGreen}{rgb}{0.1,0.5,0.1}
\definecolor{DarkRed}{rgb}{0.5,0.1,0.1}
\definecolor{DarkBlue}{rgb}{0.1,0.1,0.5}
\hypersetup{
    unicode=false,          
    pdftoolbar=true,        
    pdfmenubar=true,        
    pdffitwindow=false,     
    pdfnewwindow=true,      
    colorlinks=true,       
    linkcolor=DarkGreen,    
    citecolor=DarkGreen,    
    filecolor=DarkGreen,    
    urlcolor=DarkBlue,
    linktocpage=true,
}

\usepackage{natbib}
\setcitestyle{authoryear,round,citesep={;},aysep={,},yysep={;}}

\definecolor{RoyalBlue}{RGB}{0,100,170}
\definecolor{peach}{rgb}{1, 0.56, 0.56}
\definecolor{midgray}{RGB}{150,150,150}
\definecolor{EasternBlue}{RGB}{37,150,190}
\definecolor{sand}{RGB}{250,150,120}
\definecolor{grass}{RGB}{120, 190, 50}
\definecolor{sky}{RGB}{50,150,250}
\definecolor{Orange}{RGB}{250,150,50}
\definecolor{Cerulean}{RGB}{80,150,220}
\definecolor{Emerald}{RGB}{62,156,94}
\definecolor{Rouge}{RGB}{250,95,95}

\definecolor{ColorDef}{RGB}{80, 180, 150}
\definecolor{RevisionRed}{RGB}{240,35,35}
\definecolor{Confirm}{RGB}{103, 162, 229}
\definecolor{TODO}{RGB}{250,150,50}
\definecolor{Future}{RGB}{198, 121, 235}

\definecolor{myGold}{RGB}{231,141,20}
\definecolor{myBlue}{rgb}{0.19,0.41,.65}
\definecolor{myPurple}{RGB}{175,0,124}
\definecolor{Gred}{RGB}{219, 50, 54}
\definecolor{Ggreen}{RGB}{60, 186, 84}
\definecolor{Gblue}{RGB}{72, 133, 237}
\definecolor{Gyellow}{RGB}{247, 178, 16}
\usepackage{amsmath,amsfonts,bm}
\usepackage{cleveref}
\usepackage{thmtools}
\usepackage{dsfont}
\usepackage{stackrel}
\usepackage{algorithm,algorithmic}

\newtheorem{theorem}{Theorem}

\newtheorem{proposition}{Proposition}

\newtheorem{definition}{Definition}
\newtheorem{remark}{Remark}

\makeatletter

\newcommand{\newreptheorem}[2]
{\newenvironment{rep#1}[1]
	{\def\rep@title{#2 \ref{##1}} \begin{rep@theorem}}%
		{\end{rep@theorem}}}
\makeatother
\newreptheorem{theorem}{Theorem}
\newreptheorem{lemma}{Lemma}
\newreptheorem{corollary}{Corollary}
\newreptheorem{proposition}{Proposition}

\crefname{hypothesis}{Hypothesis}{Hypothesis}
\crefname{condition}{Condition}{Conditions}
\crefname{distribution}{Distribution}{Distributions}

\crefname{thm}{Theorem}{Theorems}

\newcommand{\eps}{\varepsilon}

\def\1{\bm{1}}

\DeclareMathAlphabet{\mathsfit}{\encodingdefault}{\sfdefault}{m}{sl}
\SetMathAlphabet{\mathsfit}{bold}{\encodingdefault}{\sfdefault}{bx}{n}

\newcommand{\N}{\mathbb{N}}

\DeclareMathOperator*{\argmax}{arg\,max}

\def\abs#1{\left| #1 \right|}

\newcommand*{\prob}{\mathbb{P}}

\newcommand{\verifier}{\mathsf{V}}
\newcommand{\lm}{\mathsf{LM}}
\newcommand{\backtrackQuota}{Q}
\newcommand{\backtrackStride}{B}
\newcommand{\prompt}{x}
\newcommand{\distrib}{\mathcal{D}_{\mathsf{Dyck}}}
\newcommand{\probSquareBracket}{p}
\newcommand{\squareOpen}{\texttt{[}}
\newcommand{\squareClosed}{\texttt{]}}
\newcommand{\roundOpen}{\texttt{(}}
\newcommand{\roundClosed}{\texttt{)}}
\newcommand{\acc}{\text{Acc}}
\newcommand{\algoName}{Tokenwise rejection sampling with backtracking }
\newcommand{\temperature}{T}

\newcommand{\dyck}{\mathsf{Dyck}}
\newcommand{\depth}{d}
\title{
On the Query Complexity of Verifier-Assisted Language Generation
}

\author{
    Edoardo Botta$^1$\footnotemark[1]
    \quad Yuchen Li$^1$\footnotemark[1]
    \quad Aashay Mehta$^1$\footnote{
    Equal contribution.
    Correspond to: Yuchen Li \texttt{yuchenl4@cs.cmu.edu} and
    Andrej Risteski \texttt{aristesk@andrew.cmu.edu}
    }  \\
    \quad Jordan T. Ash$^2$ 
    \quad Cyril Zhang$^2$
    \quad Andrej Risteski$^1$ \\
 \normalsize{
    $^1$Carnegie Mellon University
    \qquad $^2$Microsoft Research NYC}
}

\date{}

\begin{document}

\maketitle

\begin{abstract}

Recently, a plethora of works have proposed inference-time algorithms (e.g. best-of-n), 
which incorporate verifiers to assist the generation process.
Their quality-efficiency trade-offs have been empirically benchmarked on a variety of constrained generation tasks,
but the algorithmic design landscape is still largely poorly understood. In this paper, we develop a mathematical framework for reasoning about constrained generation using a pre-trained language model generator oracle and a process verifier---which can decide whether a prefix can be extended to a string which satisfies the constraints of choice. We show that even in very simple settings, access to a verifier can render an intractable problem (information-theoretically or computationally) to a tractable one. In fact, we show even simple algorithms, like tokenwise rejection sampling, can enjoy significant benefits from access to a verifier. Empirically, we show that a natural modification of tokenwise rejection sampling, in which the sampler is allowed to ``backtrack'' (i.e., erase the final few generated tokens) has robust and substantive benefits over natural baselines (e.g. (blockwise) rejection sampling, nucleus sampling)---both in terms of computational efficiency, accuracy and diversity.~\footnote{
Our codes are released at 
\url{https://github.com/YuchenLi01/LM_Query_Complexity}
} 

\end{abstract}

\section{Introduction}
\label{sec:intro}

The fast-evolving area of inference-time algorithms concerns itself with leveraging the already-impressive capabilities of language models \citep{raffel2020exploring, brown2020language, touvron2023llama}, %However, the quality of the generated samples can sometimes be suboptimal \citep{krishna2021hurdles, hendrycks2021measuring}.
%To alleviate such unreliability, a recent line of works propose to incorporate a 
together with a \emph{verifier} which can score generations of the language model. In the simplest form, called \emph{best-of-N}, the language model generates $N$ candidate responses,
which are then scored by the verifier,
and the highest-scored candidate response is chosen as the output of the inference process
\citep{cobbe2021verifiers, nakano2022webgpt}. If the verifier can score partial generations (sometimes called \emph{process reward}), the space for inference-time algorithms gets much richer: e.g., the final answer can be generated incrementally, using the verifier to guide the process (e.g., by incremental (blockwise) best-of-N, or more complicated strategies like Monte-Carlo-Tree-Search \citep{Browne2012ASO, hao2023reasoning}). Importantly, though a flurry of recent papers consider ``scaling laws'' of natural strategies, the algorithm design space of verifier-aided inference-time algorithms is still opaque.  In particular, the \emph{value of a verifier}---and the \emph{relationship it needs to have to the generator} is not well understood.     

In this paper, we show that a good verifier can substantially (both in theory and in practice) decrease the computational cost of natural generation tasks, using a pre-trained language model as an \emph{oracle}. In particular, we show that: 
\begin{itemize}
    \item Even simple \emph{constrained generation tasks}---where we are trying to generate a string in the support of a language oracle, subject to some structural constraint (e.g. describable as a simple formal language, like a regular language)---can be \emph{computationally intractable in the absence of a verifier}. 
    \item Conversely, access to a good \emph{process verifier}, one that can decide whether a prefix can be completed to a constraint-satisfying string, can remove these intractabilities. Moreover, even simple algorithms like tokenwise rejection sampling---wherein we generate the string one token at a time, using the process verifier as a means to accept or reject---can have substantive computational benefits over the baseline of rejection sampling. 
    \item Finally, on natural constrained generation tasks---namely, generating test cases for Python functions with a pretrained CodeLlama \citep{roziere2023code}---a \emph{verifier can be trained}, such that a simple, but natural generalization of tokenwise rejection sampling which is allowed to ``backtrack'' the last few generated tokens, achieves substantial benefits in computational efficiency, accuracy, and diversity of the generations.
\end{itemize}

\section{Setup and notation}
\label{sec:theory:setup} 

Throughout, we let $\Sigma$ be a nonempty finite set, denoting the vocabulary. We denote as $\Sigma^i$ the set of strings of length $i$ and  by $\Sigma^{*} = \cup_{i \in \mathbb{N}} \Sigma^i$ the set of all finite strings on $\Sigma$. Given a string $s \in \Sigma^{*}$
, we denote as $s_i$ its $i$-th element and as $s_{i:j}$ the substring of $s$ starting at its $i$-element and ending at its $j$-element, included. 
We use 
$\abs{s}$ to denote the length of string $s$ and
$\epsilon$ to denote the empty string. 
Finally, we let $x \circ y$ denote the concatenation of string $x$ followed by string $y$.

\begin{definition} [Autoregressive oracle]
\label{theory:autoreg_gen_def}
     An \textbf{autoregressive oracle} $\mathcal{O}$ takes as input a string $s \in \Sigma^{*}$ and returns a sample from a next-token distribution $\mathcal{O}(s): \Sigma \to \mathbb{R}^+$. 
     
     We will denote the corresponding joint distribution over strings $s \in \Sigma^*$ as $p_{\mathcal{O}}:\Sigma^* \to \mathbb{R}^+$. 
     Correspondingly, $\forall s \in \Sigma^{*}$, let $p_{\mathcal{O}}(\cdot \mid s)$ denote the distribution over completions of $s$ predicted by $\mathcal{O}$.
\end{definition}
% \yuchen{ (addressed)
% The above definition does not reflect the ``autoregressiveness".
% In particular, to be compatible with \Cref{def:oracle_complexity},
% each call to $\mathcal{O}$ has to generate one single token, instead of a whole string.
% }
\begin{definition}[Constrained generation]
\label{theory:constrained_gen_def}
    \textbf{Constrained generation} with respect to an oracle $\mathcal{O}$, a constraint set $A$, and vocabulary $\Sigma$ is the task of producing an element $s \in A \subseteq \Sigma^{*}$ such that  $p_{\mathcal{O}}(s) > 0$. If no such $s$ exists, the algorithm needs to output \texttt{FAIL}. %denotes the restriction of $p_{\mathcal{O}}(s) := p(s_1)p(s_2|s_1) \dots p(s_{|s|}|s_{1:|s|-1})$ to $A$, $p_{\mathcal{O}} |_A(s) := p_{\mathcal{O}}(s) \mathbf{1}_A(s)$
\end{definition}
When not clear from context, we will specify instances of this task by the triple $(\Sigma, A, \mathcal{O})$. Under suitable choices of the vocabulary $\Sigma$ and the target domain $A$, one recovers several language modeling tasks of theoretical and practical relevance as special cases of constrained generation. Specifically, our experiments consider the tasks of generating (i) valid strings under the Dyck grammar (Section \ref{sec:experiments:synthetic}) and (ii) valid test cases for a given Python functions (Section \ref{sec:experiments:codellama}), where the oracles return samples from an appropriately pretrained language model. We recover these tasks from Definition \ref{theory:constrained_gen_def} by setting:
\begin{itemize}
    \item (i) $\Sigma$ as the set of open and close parentheses, and $A$ as the set of valid sequences of given length.
    \item (ii) $\Sigma$ as a set of characters from the Unicode standard (possibly after tokenization) and $A$ as the set of strings that are valid test cases for an input function in the Python programming language.
    %\item \textbf{JSON mode}: If $\Sigma$ is the Unicode standard (possibly tokenized via e.g. byte-pair encoding \cite{sennrich2015neural}) and $A$ is the set of valid JSON strings.
    %\item \textbf{Lexical constraints}: If $\Sigma$ is the Unicode standard  and $A$ is the set of sentences satisfying some lexical constraints.
\end{itemize}
Note that this task is easier than the task of sampling according to the \emph{restricted distribution} $p(s) \propto \mathbf{1}(s \in A) p_{\mathcal{O}}(s)$, which asks that the relative weights of the strings $s \in A$ that are generated match the probabilities assigned by $p_{\mathcal{O}}$. However, in many settings---e.g., generating a proof of a mathematical problem, or code that performs some intended functionality---we merely care about producing one good sample.    

We will be considering ``process verifiers'' that take as input a prefix $s$, and output whether or not such a prefix can be completed to a string $s \circ s' \in A$. This is a natural formalization of a ``process reward'', as it assigns a belief to a partial generation. In the theoretical results (Section~\ref{s:hardness} and \ref{s:plusverifier}), we'll assume access to such an idealized verifier. In the empirical results (Section~\ref{sec:experiments}), such a verifier will be trained and will output a value between 0 and 1, which can be naturally interpreted as a probability that the prefix $s$ is completable to a string $s \circ s' \in A$.  
\begin{definition}[Process verifier]
\label{def:verifier}
    Given a constraint set $A$,
    a verifier is a function $V: \Sigma^{*} \to \{0,1\}$ such that 
    $\forall s \in \Sigma^*$,
    $V(s) = 1$ if and only if 
    $\exists s' \in \Sigma^*$ 
    such that $s \circ s' \in A$.
\end{definition}
% \yuchen{We need a definition for a prefix verifier, rather than only a full-sequence verifier}
% Using the notion of a verifier, constrained generation is alternatively formulated as sampling from the distribution over length $d$ strings
% \begin{align*}
%     p'(s) \propto p(s)V(s),
% \end{align*}
% where $V(s) = \mathbf{1}_A(x_{1:d})$. This is effectively a special case of a reward model [add citation] with binary rewards.
Designing algorithms given access to oracles which perform certain tasks, is a classical tool in computer science (this is the basis of Turing reductions in computational complexity), as well as optimization (e.g., zero-order optimization assumes a value oracle for a function, first-order optimization a gradient oracle, etc.) In the context of generative modeling, analyses based on oracle complexity have been carried out in the settings of diffusion models, where sampling algorithms rely on score oracles \cite{chen2022sampling}.

We will consider several natural algorithms that use an autoregressive oracle and a (process) verifier: 

\begin{definition}[Rejection sampling] 
\label{d:rejection}
Rejection sampling works by repeatedly generating a string $s$ according to $p_{\mathcal{O}}$, then running a verifier $V$ on the complete string---and accepting when the verifier outputs $V(s) = 1$.    
\end{definition}

Note, this algorithm only needs a verifier that decides the membership in $A$, rather than a process verifier. On the other hand, because the entire string needs to be generated first before being verified---the number of generations until the verifier accepts is likely very large.  

\begin{definition}[Tokenwise rejection sampling] 
\label{d:tokenwise}
Tokenwise rejection sampling works by generating a string one token at a time. To generate the next token $t$, given a prefix $s$, we sample $t \sim \mathcal{O}(s)$, 
% \yuchen{
% Slight mismatch in notation:
% In \Cref{theory:autoreg_gen_def},
% $\mathcal{O}(s)$ is next-token distribution,
% whereas $p_{\mathcal{O}}$ is completions.
% }
and run the process verifier on $V(s \circ t)$. We repeat this, until $V(s \circ t) = 1$, then proceed to the next token.  
\end{definition}

This algorithm requires a process verifier. However, since a partial string is accepted only if the process verifier accepts, the number of generations needed is likely to be smaller. In fact, we provide a very simple example in Section~\ref{s:plusverifier}.  

Finally, we consider a ``backtracking'' strategy, in which the model is allowed to erase some of its generations. The reasons to consider such a strategy is to allow the model to get ``unstuck'': if the process verifier decides the current prefix cannot be completed to a valid string in $A$, it is possible that erasing the last few tokens will make it easier for the model to correct its mistake, compared to erasing just the last token.
More formally, the framework of our algorithm is given by \Cref{alg:sampling_with_backtracking} below.~\footnote{
The algorithm is a bit more involved, so we will describe it in pseudocode rather than text.
Besides the notations in \Cref{sec:theory:setup},
\Cref{alg:sampling_with_backtracking} uses the following additional common conventions: 
\texttt{<eos>} denotes the end-of-sequence token;
$s_{\abs{s}} \ne \texttt{<eos>}$ is understood as True when $s = \eps$;
for any starting index $i$ and ending index $j$, 
if $i > j$,
then $s_{i:j} = \eps$.
In line 10, why redoing the erased positions using argmax:
our results in \Cref{sec:experiments:synthetic:dyck} suggests that \emph{out-of-distribution prefix} is a cause of generator mistakes.
As a remedy, redoing the erased positions using argmax is intended to increase the generator-predicted probability of the currently sampled prefix.
We include an ablation study in \Cref{sec:appendix:experiments:algo} verifying that this improves the accuracy.
}

\begin{algorithm}
\caption{\algoName}
\label{alg:sampling_with_backtracking}
\begin{algorithmic}[1]
\STATE \textbf{Input:} Prompt $\prompt$, generator $\mathcal{O}$, verifier $\verifier$, length $D \in \N_+$, backtrack quota $\backtrackQuota \in \N$, backtrack stride $\backtrackStride \in \N_+$
\STATE $s \gets \epsilon$
\WHILE{$\abs{s} < D$ and $s_{\abs{s}} \ne \texttt{<eos>}$}
    \STATE Sample $\hat{s} \sim \mathcal{O}(x \circ s)$
    \STATE $s \gets s \circ \hat{s}$
    \IF{ $Q > 0$ and $\verifier(x \circ s) = 0$}
        \STATE $s \gets s_{1 : \abs{s} - B}$
        \STATE $Q \gets Q - 1$
        \FOR{i in $1 \cdots B$}
            \STATE Choose $\hat{s} \in \argmax \mathcal{O}(x \circ s)$
            \STATE $s \gets s \circ \hat{s}$
        \ENDFOR
    \ENDIF
\ENDWHILE
\end{algorithmic}
\end{algorithm}

When arguing about lower bounds, a natural lower bound on the complexity of an algorithm is the number of oracle calls needed\footnote{In our case, the number of calls is a randomized quantity, so a natural quantity to consider is the expected number of oracle calls. It is of course reasonable to consider finer-grained notions like tail bounds on the number of calls.}, particularly so when this dominates the cost of the algorithm, as is frequently the case for language models: 
\begin{definition}[Oracle complexity]
    \label{def:oracle_complexity}
    Given a (possibly randomized) algorithm $\mathcal{A}$ that solves the constrained generation instance $(\Sigma, A, \mathcal{O})$, the \textbf{oracle complexity} of $\mathcal{A}$ is defined as the expected number of calls to the oracle made by $\mathcal{A}$ to solve $(\Sigma, A, \mathcal{O})$, namely:
    \begin{align*}
        \mathcal{C}(\mathcal{A}) = \mathbb{E}[\# \text{calls to} \ \mathcal{O} \ \text{made by running} \ \mathcal{A}],
    \end{align*}
    where the expectation is taken over the randomness of the oracle $\mathcal{O}$ and the randomness of the algorithm $\mathcal{A}$.
\end{definition}

%Specifically, this notion fits well the practice of language model sampling, where the runtime is dominated by the computation of the distribution over the next-token.~\footnote{
%See \Cref{sec:appendix:discussions} for additional discussions.
%} 

Finally, we recall the classical knapsack problem,  which will be used in a reduction to prove computational intractability results for the constrained generation task: 
\begin{definition}[Knapsack problem]
Given a set of weights $\{X_i \in \mathbb{Z}_{\geq 0} \mid i \in [D]\}$ and $c \in \mathbb{Z}_{\geq 0}$, the knapsack problem seeks an assignment of the variables $(a_i)_{i=1}^D$, with $a_i \in \{0,1\} \ \forall i \in [D]$ such that $c = \sum_{i=1}^{D} a_i X_i$.

%Further, we denote by $\Omega(\{X_i\}_{i \in [D]}, c)$ the set of all assignments $(a_i)_{i=1}^D$ solving the knapsack instance specified by $\{X_i\}_{i \in [D]}, c$.   
\label{d:knapsack}
\end{definition}

The problem is (weakly) NP-hard, even for some very special choices of $c, X_i$. 
\section{Constrained generation is hard without a verifier}
\label{s:hardness}

First, we show that the constrained generation task (Definition~\ref{theory:constrained_gen_def}), without access to a process verifier can be intractable---even if the constraint set $A$ is extremely simple (e.g. the parity of a binary string). 

The source of intractability can be \emph{information-theoretic}: namely, if the oracle does not have a succinct description, the algorithm may need to query it prohibitively many times to identify what oracle it's interacting with. We view this as a plausible obstruction in practice as well: language models frequently behave unpredictably ``in-the-tails'', which becomes increasingly more likely when generating long strings. Thus, to inspect the behavior of the model on long strings, many queries are needed. 

The source of the intractability can also be \emph{computational}: namely, even if the oracle is very simple (e.g., a uniform distribution), generating a member of $A$ can be NP-hard, even if checking membership in $A$ can be done efficiently. Perhaps this should not come as a surprise: after all, easy verification of membership, but hard generation, is the hallmark of NP-hard problems.     

Proceeding to the first result, we show the following: 

\begin{theorem} 
\label{thm:info_theoretical_lower_bound}
There exists a constrained generation task $(\Sigma, A, \mathcal{O})$ for which $\Sigma = \{0,1\}$, $A \subseteq \Sigma^D$, and $\mathcal{O}$ is an (unknown) member of a set of $2^{D-1}$ possible oracles, such that any (possibly randomized) algorithm $\mathcal{A}$ has an (expected) oracle complexity of at least $2^{D-1}$. 
\end{theorem}
Intuitively, the lower bound is shown by engineering a scenario such that the behavior of the oracle on long strings is unknown to the algorithm---but success of the generation task relies on ``guessing'' this behavior correctly.  
%The proof argument considers a family of oracles indexed by strings of length $D-1$, which sample digits uniformly except for the last position and the unique digit that would solve the task is sampled if and only if the current $D-1$ long prefix matches the string indexing the oracle $\hat{s}$. 

% TODO: proof moved to appendix, can move back later
% The proof is in \Cref{sec:appendix:proof:info_theoretical_lower_bound}.
\begin{proof}
    Consider the constrained generation task $(\Sigma, A, \mathcal{O}_{\hat{s}})$, such that $\Sigma := \{0,1\}$, $A := \{s \in \Sigma^D: \sum_{i=1}^{D} s_i \mod 2 = 0 \}$ for some fixed $D \in \mathbb{Z}_{+}$. Moreover, the oracle $\mathcal{O}_{\hat{s}}$ 
    is indexed by an (unknown to the algorithm) $\hat{s} \in \Sigma^{D-1}$, and it specifies the autoregressive distribution defined s.t.
    $\forall s \in \Sigma^*, |s| < D-1$, we have $p_{\mathcal{O}_{\hat{s}}}(1 | s) = p_{\mathcal{O}_{\hat{s}}}(0 | s) = 1/2$; while for $s \in \Sigma^*, |s| = D-1$, it satisfies:    
    
    $\forall s \neq \hat{s} \in \Sigma^{D-1},  s_D \in \{0,1\}$, we have:
    \begin{align}
    \label{eq:not_s_hat}
    p_{\mathcal{O}_{\hat{s}}}(s_D \mid s) &=
    \begin{cases}
        1, \ \text{if} \ \left( \sum_{j=1}^{D-1} s_j + s_D \right) \mod 2 = 1 \\
        0, \ \text{otherwise}
    \end{cases}
    \end{align}
    
    For $s = \hat{s},  s_D \in \{0,1\}$, we have:    
    \begin{align}
    \label{eq:is_s_hat}
    p_{\mathcal{O}_{\hat{s}}}(s_D \mid s) &= 
    \begin{cases}
        1, \ \text{if} \ \left( \sum_{j=1}^{D-1} s_j + s_D \right) \mod 2 = 0 \\
        0, \ \text{otherwise}
    \end{cases}
    \end{align}

    Suppose first that the algorithm is deterministic, and we choose the prefix $\hat{s}$ uniformly at random. Let us denote by $x_1, x_2, x_3, \dots, x_{q} \in \Sigma^*$ the queries to $\mathcal{O}$ generated by the algorithm. The claim is that expected number of queries $q$ needed to ensure at least one $x_i, i \in [q]$ is in $A$ is $2^{D-1}$. Indeed, the $x_i$ s.t. $|x_i| < D-1$ reveal no information about $\hat{s}$: the output of $\mathcal{O}$ is a uniform Bernoulli random variable regardless of the value of $\hat{s}$. On the other hand, if at some point the algorithm has queried a set $S$ of $x_i$ of length $D-1$, the probability over $\hat{s}$ is uniform over $\Sigma^{D-1} \setminus S$.  Hence, the expected number of queries $q$ (expectation being over the choice of $\hat{s}$) a deterministic algorithm needs is lower bounded by $2^{n-1}$. 

    By Yao's minimax lemma \citep{yao1977probabilistic}, this means that for any (even possibly randomized) algorithm $\mathcal{A}$, there exists $\hat{s}$ on which the algorithm makes at least $2^{n-1}$ queries in expectation.  
    %By way of contradiction, suppose that an algorithm with less than $2^{D-1}$ queries succeeds in solving the constrained generation task [Andrej: Is task defined to generate a *single* sample?]. Suppose the queries generated by the algorithm are $x_1, x_2, x_3, \dots, x_{q} \in \Sigma^*$, for $q \leq 2^{D-1}$. 
    %The total number of distinct prefixes of length $D-1$ queried by the algorithm is thus less than $2^{D-1}$. Thus, there exists a $y \in \Sigma^{D-1}, \mbox{s.t. } \forall i \in [q], y \neq x_i$ [Andrej: this is a bit handwavy --- the lengths of $x_i$'s are arbitrary. What do we do with the shorter queries?]. Then, if the algorithm was interacting with oracle $\mathcal{O}_y$, it would    
    %The task effectively reduces to sampling from a point-mass centered at $\hat{s}$. In the worst case where $\hat{s}$ is chosen adversarially based on the $D-1$ long prefix produced by the generation algorithm, all possible $D-1$ long prefixes are explored before finding the matching one, requiring $2^{D-1}$ queries to the oracles are required   
\end{proof}

%\subsection{Reduction from the Knapsack problem}
%\label{sec:theory:knapsack}
%Our first theoretical contribution is a computational lower bound on the oracle complexity of constrained generation. By a reduction argument, our proof shows that constrained generation from an arbitrary known oracle is at least as hard as modular knapsack. \ar{Compare the two lower bounds: in one the hardness comes from not knowing what the oracle does on long strings; the other from the interaction of $\Sigma$ and $A$}

\begin{remark}
    To help readers parse our proof above, we provide its informal intuition.
    The oracle $\mathcal{O}_{\hat{s}}$ can be thought of as hiding a secret message $\hat{s}$ which is a binary sequence of length $D-1$.
    Because of our construction in \Cref{eq:not_s_hat},
    by querying the oracle with any string $s \ne \hat{s}$, 
    the output of the oracle will not reveal any information about $\hat{s}$.
    Therefore, to know anything about $\hat{s}$, 
    the query $s$ needs to exactly match $\hat{s}$.
    For any deterministic order of searching for $\hat{s}$ over $\{0,1\}^{D-1}$, 
    the worst case is always that $\hat{s}$ is the last item in the search order, causing runtime $2^{D-1}$.
    
    Moreover, \Cref{thm:info_theoretical_lower_bound} generally applies to any algorithm $\mathcal{A}$. 
    In particular, $\mathcal{A}$ is even allowed to \emph{never} query the generator oracle at all.
    Intuitively, one candidate counter-example of our theory would be 
    a simple algorithm $\mathcal{A}$ which always outputs 0 at all positions
    (as this will obviously satisfy the counstraint $A := \{s \in \Sigma^D: \sum_{i=1}^{D} s_i \mod 2 = 0 \}$).
    However, recall that \Cref{theory:constrained_gen_def} requires that 
    the output must have nonzero probability under the generator oracle $\mathcal{O}_{\hat{s}}$.
    Note that $s_i = 0 \; \forall i \in [D]$ will not have nonzero probability under $\mathcal{O}_{\hat{s}}$, 
    thus violating the constraints
    (unless $\hat{s}_i = 0 \; \forall i \in [D-1]$ ).
    The intuitive reason why the above counter-example does not work is that, 
    it is necessary to use the oracle $\mathcal{O}_{\hat{s}}$ to know the $\hat{s}$ that it hides. 
    Therefore, any algorithm $\mathcal{A}$ is governed by the above-mentioned lower bound of searching for $\hat{s}$ by querying $\mathcal{O}_{\hat{s}}$.
\end{remark}

Proceeding to the computational lower bound, the theorem we show is as follows:
% (proof is in \Cref{sec:appendix:proof:computational_lower_bound}) 

\begin{theorem}
\label{thm:computational_lower_bound}
There exists a constrained generation task $(\Sigma, A, \mathcal{O})$ for which $\Sigma = \{0,1\}$, membership in $A \subseteq \Sigma^D$ can be checked in time polynomial in $D$, and $\mathcal{O}$ is such that $\forall s \in \{0,1\}^D, p_{\mathcal{O}}(s) > 0$, the generation task is NP-hard. \end{theorem}

% TODO: proof moved to appendix, can move back later
\begin{proof}
    We construct a reduction from the knapsack problem (Definition~\ref{d:knapsack}). Let the set $\{X_1, \ldots, X_D\}$ and the integer $c$ specify an arbitrary instance of the knapsack problem. Consider
    the constrained generation task specified by $\Sigma := \{ 0, 1\}$, $A := \{ s \in \Sigma^D: \forall i \in [D], s_i \in \{ 0, 1\}; \ \sum_{i=1}^D s_i X_i = c\}$. Membership in this $A$ can be clearly verified in polynomial time.  %$\mathcal{O}$ be an oracle such that $p_{\mathcal{O}}(s) = 1/2^D, \forall s \in \Sigma^D$.
     Suppose we have a poly-time algorithm that generates a solution $\hat{s}$ to $(\Sigma, A, \mathcal{O})$. Since $\forall s \in \Sigma^D, p_{\mathcal{O}}(s) > 0$, $\hat{s}$ provides a solution to the knapsack problem, as we needed.  
\iffalse satisfies
\begin{align}
\label{theory:mod_knapsack_constraint}
\sum_{i=1}^D \hat{s}_i \mod X_0 = c. 
\end{align}
Then, let 
\begin{align}
   a_i = \mathbf{1}_{\{ X_i \}}(\hat{s}_i) \ \forall i \in [D].
\end{align}
The assignment $(a_i)_{i=1}^D$ satisfies 
\begin{align}
    \sum_{i=1}^D a_i X_i = \sum_{i=1}^D \mathbf{1}_{\{ X_i \}}(\hat{x}_i) X_i = \sum_{i=1}^D \hat{s}_i
\end{align}
and, by equation \ref{theory:mod_knapsack_constraint}, it is a solution to modular knapsack.
Suppose, instead, that we have a poly-time algorithm to find $(a_i)_{i=1}^D$ satisfying modular knapsack. Let 
\begin{align}
\hat{s}_i = a_iX_i \ \forall i \in [D]. 
\end{align}
Then, $\hat{s}$ satisfies equation \ref{theory:mod_knapsack_constraint} and $\hat{s}_i \in \{0, X_i\} \ \forall i \in [D]$, which implies that it is a sample satisfying the constrained generation instance $(\Sigma, A, \mathcal{O})$.
\fi
\end{proof}

%\subsection{Information-theoretic (oracle) lower bound}
%The lower bound from section \ref{sec:theory:knapsack} considers an instance of constrained generation where the oracle is fully known and the source of hardness lies in the choice of $\Sigma$ and $A$. We now state a lower bound on the number of queries to the oracle $\mathcal{O}$ needed to solve the constrained generation task---even if the constraint set $A$ is very simple (specifically, binary strings with even number of ones). The setting we consider differs from the computational bound. The assumed input structure implies the source of hardness to be that the behavior of the oracle on long strings is now unknown to the algorithm.   

\section{Constrained generation with process verifier gets easier}
\label{s:plusverifier}

While pessimistic, the message of Section~\ref{s:hardness} agrees with recent developments in inference-time scaling: namely, many natural tasks of interest seem to require a verifier to be solved. 

First, we show that the simplest ``natural'' algorithm with a process verifier, tokenwise rejection sampling (Definition~\ref{d:tokenwise}), can be much more efficient (exponentially so) in terms of oracle complexity compared to the trivial baseline of rejection sampling (Definition~\ref{d:rejection}). 

\begin{proposition} 
\label{prop:easy_with_verifier}
Consider the constrained generation task $(\Sigma, A, \mathcal{O})$, s.t. $\Sigma = \{0,1\}$, $A = \{0^D\}$ and $\mathcal{O}$ is uniform over $\Sigma^D$. Then: 
\begin{enumerate}
    \item The expected oracle complexity of rejection sampling (Definition~\ref{d:rejection}) is $2^D D$.  
    \item The expected oracle complexity of tokenwise rejection sampling (Definition~\ref{d:tokenwise}) with a perfect process verifier is $2 D$.  
\end{enumerate}
\end{proposition}
% TODO: proof moved to appendix, can move back later
% The proof is in \Cref{sec:appendix:proof:easy_with_verifier}.
\begin{proof}
Both claims are straightforward. (1) follows as generating one guess for the string $s$ takes $D$ oracle calls. Moreover, the probability of the full string matching the only string in $A$ (i.e., $0^D$) is $1/2^D$. As the number of calls to generate $0^D$ is a geometric random variable, the expected number of full string generations is $2^D$. 

For (2), since $\mathcal{O}$ is uniform, at each token, the probability of drawing $0$ is $1/2$. Hence, the expected number of calls per coordinate needed is $2$ --- making the total number of expected calls for the entire string $2D$. 
\end{proof}

This proposition underscores the power of a process verifier --- even in extremely simple settings, and even when used in conjunction with a very simple algorithm. 

In fact, one can easily see that with a perfect process verifier, one can easily solve the constrained generation task with $|\Sigma| D$ calls: at each position, one queries the process verifier for each possible continuation of the string, and accepts only if the process verifier accepts. Of course, in practice, the verifier is not perfect, and its accuracy likely depends on how ``out-of-distribution'' the prefix it's queried on is
(See \Cref{sec:experiments:synthetic:verifier_reduces_errors_unseen_ood} and \Cref{sec:experiments:codellama:ood})

We finally remark that a process verifier, as we defined it, is clearly useful to solve the generation task. If we instead wanted to sample from the restricted distribution $p(s) \propto \mathbf{1}(s \in A)p_{\mathcal{O}}(s)$, it's not clear how useful the process verifier is. For instance, if we use the simple tokenwise rejection sampling (Definition~\ref{d:tokenwise}), it's easy to see that the distribution we produce samples from is \emph{not} the restricted distribution:
% (and proof is in \Cref{sec:appendix:proof:calibration_is_hard}): 

\begin{proposition}
\label{prop:calibration_is_hard}
Consider the constrained generation task $(\Sigma, A, \mathcal{O})$, s.t. $\Sigma = \{0,1\}$, $A = \{s \in \Sigma^D: \exists i \in [D], s_i = 0\}$ and $\mathcal{O}$ is uniform over $\Sigma^D$. Then, tokenwise rejection sampling does \emph{not} produce samples from $p(s) \propto \mathbf{1}(s \in A)p_{\mathcal{O}}(s)$.      
\end{proposition}

% TODO: proof moved to appendix, can move back later
\begin{proof}
By Definition~\ref{d:tokenwise}, until the last token is being generated, the process verifier will always accept (as there exists a string with at least one 0 coordinate in the coordinates that haven't yet been sampled). Now, for the prefix $1^{D-1}$, the only completion that is in $A$ is $1^{D-1} \circ 0$. This means that $1^{D-1} \circ 0$ is assigned probability mass $\frac{1}{2^{D-1}}$ under the tokenwise rejection sampling schema. All other strings in $\Sigma^D$ are assigned a probability $\frac{1}{2^D}$. On the other hand, $p(s) \propto \mathbf{1}(s \in A)p_{\mathcal{O}}(s)$ assigns uniform mass on all strings in $A$ --- proving the claim of the proposition.     
\end{proof}

\section{\hspace{-2.5mm}Backtracking: a surprisingly effective rejection sampling strategy}
\label{sec:experiments}

The flexibility of the tokenwise rejection sampling with backtracking (\Cref{alg:sampling_with_backtracking}) makes it a very natural strategy to use in conjunction with trained verifiers. We perform a thorough empirical investigation into the applicability of \algoName in constrained language generation,
and benchmark it against common baselines,
including rejection sampling (\Cref{d:rejection}),
nucleus sampling \citep{holtzman2020the},
temperature scaling,
and ``block best-of-N" (\Cref{sec:experiments:codellama:verifier}) sampling,
on both synthetic data (\Cref{sec:experiments:synthetic}) and more realistic data (\Cref{sec:experiments:codellama}).
We observe that across various settings, 
\algoName reduces query complexity, improves accuracy, and does not hurt diversity.

\subsection{Language models trained on synthetic data}
\label{sec:experiments:synthetic}

\subsubsection{Dyck grammar as a sandbox}
\label{sec:experiments:synthetic:dyck}

Real-world LLM pretraining data \citep{li2024datacomp} typically involves many diverse structures, 
so when an LLM algorithm outperforms baselines on a benchmark, it is generally challenging to precisely identify which component of the algorithm improved the handling of which structures of the data.

To have a quantitative control over the structure in the pretraining data distribution,
and to derive fine-grained observations about the effects of \algoName,
we synthetically generate the pretraining data based on the \emph{Dyck grammar}~\citep{SCHUTZENBERGER1963246},
a classic formal language (context-free grammar) consisting of balanced parentheses of multiple types
(for example, ``$[()]$'' is valid but ``$([)]$'' is not). 
Dyck serves as a useful sandbox, as it typifies features such as long-range dependencies and a hierarchical, tree-like structure—characteristics often found in both natural and programming language syntax—and has been a subject of interest in numerous theoretical studies on Transformers~\citep{yao2021self,liu2023same,liu2023Transformers,wen2023uninterpretability}.
More formally:

\begin{definition}[Dyck distribution]
\label{def:dyck}
    $\dyck_D$ denotes the Dyck language~\footnote{
        We follow a simplified version of \citet{wen2023uninterpretability} in defining a probability distribution over strings in a Dyck language.
    } 
    of length $D$ defined over  
    the vocabulary $\Sigma = \{ \squareOpen, \squareClosed, \roundOpen, \roundClosed \}$,
    whose \emph{length-$N$ prefix set} is denoted as $\dyck_{N}, \forall N \in [D]$.
    For a valid prefix $w_{1:N} \in \dyck_{N}$, the \emph{depth} of $w_{1:N}$ is 
    \begin{align*}
        \depth(w_{1:N}) &= \# \text{Open Brackets in }w_{1:N} \\
        &-  \# \text{Close Brackets in }w_{1:N}.
    \end{align*}
    The distribution $\distrib$ over $\dyck_{N}$, 
    (parameterized by $\probSquareBracket, q \in (0,1)$)
    is defined such that $\forall w_{1:N} \in \dyck_{N}$,
    \begin{align}
        \label{eq:dyck}
        &\prob(w_{1:N})
        \propto
        \probSquareBracket^{\abs{\{i \mid w_i = \squareOpen, \depth(w_{1:i}) = 1\}}} 
        \cdot (1 - \probSquareBracket)^{\abs{\{i \mid w_i = \roundOpen, \depth(w_{1:i}) = 1\}}} \\
        &\cdot (\probSquareBracket q)^{\abs{\{i \mid w_i = \squareOpen, \depth(w_{1:i}) > 1 \}}}
        \cdot ((1-\probSquareBracket) q)^{\abs{\{i \mid w_i = \roundOpen, \depth(w_{1:i}) > 1 \}}} \\
        &\cdot (1 - q)^{\abs{\{i \mid w_i \in \{ \squareClosed, \roundClosed \},  \depth(w_{1:i}) \le D - i\}}}.
        \notag
    \end{align}    
\end{definition}

\begin{remark}
    \Cref{eq:dyck} defines an intuitive autoregressive generative process for $\dyck_D$:
    if the current depth is 0, 
    then sample the next token from $\squareOpen$ and $\roundOpen$ 
    with probability $\probSquareBracket$ and $1-\probSquareBracket$ respectively;
    else if the current depth is $D - i + 1$, 
    implying that all the remaining positions have to be closing brackets,
    then deterministically close the last unmatched open bracket~\footnote{
    \label{footnote:acc}
    At any position, there is at most one valid closing bracket.
    };
    else, sample the next token from open or close brackets with probability $q$ and $1-q$ respectively.
    In other words, $\probSquareBracket$ controls the proportion of square vs. round brackets, while $q$ controls the tendency to predict an open bracket when possible 
    (a large $q$ may result in a large depth at some position).
    % ~\footnote{
    % By definition, a necessary but not sufficient condition for a string $s$ to belong to $\dyck_D$ is that $\depth(s) = 0$.
    % }
\end{remark}

In our experiments, we pretrain autoregressive Transformer \citep{vaswani2017attention} Language models 
(6 layers, 8 heads per layer, hidden dimension 512)
from scratch 
on data sampled from $\distrib$ 
with $D = 32, p = 0.2, q = 0.5$.
We use batch size 32, weight decay 0.1, learning rate 3e-4 with 100 warmup steps,
and follow \citet{block2024butterfly} to use exponential moving average to stabilize training.
We reached 100\% training and (in-distribution) validation accuracy.
% ~\footnote{
% On independently sampled strings from $\distrib$, measure whether the model correctly predicts closing brackets. (Recall \Cref{footnote:acc}.)
% } 

To search for stronger signals in benchmarking the accuracy of the trained model,
we will prompt it using the following type of \emph{out-of-distribution} prompts.
Note that since $p < 0.5$, the training data contains less square brackets than round brackets,
so long prefixes with many square brackets will be \emph{out-of-distribution} prompts for the trained model.
We generated a set of such out-of-distribution prompts $\dyck_{OOD}$ from $\dyck_{N}$ with $p = 0.8$ where the prefix length $N$ is uniformly randomly sampled from $25 \le N \le 31$.
We let the trained language model complete these prompts and check whether the completed string is in $\dyck_D$. 
Quantitatively:

\begin{definition}[Prompt completion accuracy]
    \label{def:acc}
    Given an autoregressive oracle $\mathcal{O}$ (\Cref{theory:autoreg_gen_def}) 
    and a set of prefix prompts $X$,
    the accuracy of $\mathcal{O}$ in completing $X$ is:
    \begin{equation*}
        \acc(\mathcal{O}, X) = \frac{1}{\abs{X}} \sum_{\prompt \in X, y \sim p_{\mathcal{O}}(\cdot \mid \prompt)} \1_{x \circ y \in \dyck_D}  
    \end{equation*}
\end{definition}

We construct the autoregressive oracle $\mathcal{O}_{\text{nucleus}}$ 
which predicts the next-token distribution based on our trained model 
with nucleus sampling \citep{holtzman2020the} top\_p set to 0.9.
We observed that 
$\acc(\mathcal{O}_{\text{nucleus}}, \dyck_{OOD}) = 94.23\%$.
We will show that $\mathcal{O}_{\text{verifier backtracking}}$ based on \Cref{alg:sampling_with_backtracking} can significantly reduce the remaining error rate.

\subsubsection{Training the verifier}
\label{sec:experiments:synthetic:verifier}

We collect a set of 441 prompts in $\dyck_{OOD}$ in which the trained model (denoted as $\lm$) made mistakes when completing them.
We implement a rule-based error parser according to the grammars of $\dyck_D$ which identifies the first position of error in each model completion.
Applying this parser to the model mistakes, we obtain a set of model-generated strings $X_{\text{error}} \subset \Sigma^*$ which contain errors.
By contrast, we sample another set of 441 strings $X_{\text{correct}} \sim \dyck_{OOD}$ such that
$X_{\text{error}}$ and $X_{\text{correct}}$ have the same length distribution.
We train a lightweight neural network verifier to distinguish $X_{\text{error}}$ from $X_{\text{correct}}$.

Concretely, to maximally exploit the representations learned by $\lm$,
we train a 1-linear-layer verifier $\verifier$
whose features are the last-layer-last-position representations by $\lm$ of strings in $X_{\text{error}} \cup X_{\text{correct}}$,
and labels are 0 for strings in $X_{\text{error}}$ and 1 for strings in $X_{\text{correct}}$.
Consequently, the trainable parameters of $\verifier$ are a single matrix of dimensionality 512 by 2.
Among the 882 strings in $X_{\text{error}} \cup X_{\text{correct}}$, 
we use 792 samples for training, and 90 samples for validation.
Despite being slightly over-parameterized,
this minimal verifier $\verifier$ achieved on average 93\% (with standard error 3.9\%) validation accuracy across 10 repetitions.
\Cref{fig:correct_vs_err_rep_dyck} in \Cref{sec:experiments:synthetic:visualizing} illustrates the intuition of why a lightweight verifier may be surprisingly effective with a small number of labeled samples.
% TODO: moved some results to appendix, may move back later
% In \Cref{sec:experiments:synthetic:preventing_errors} and \Cref{sec:experiments:synthetic:verifier_reduces_errors},
We next verify that forcing a backtracking at prefixes where the model made mistakes can effectively improve the completion accuracy (\Cref{sec:experiments:synthetic:preventing_errors}),
and that the trained verifier in this section can mostly catch those mistakes and thus mostly retaining the accuracy gain (\Cref{sec:experiments:synthetic:verifier_reduces_errors}).

\subsubsection{Backtracking effectively reduces errors}
\label{sec:experiments:synthetic:preventing_errors}

The trained language model $\lm$ made a mistake at the last position of each string $\prompt \in X_{\text{error}}$.
We therefore use ``error-inducing prefixes" $X_{\text{error-inducing}}$ to denote $\{ \prompt_{1:\abs{\prompt}-1} \mid \prompt \in X_{\text{error}} \}$.
\Cref{table:dyck_backtrack_at_error_inducing} shows that at prefixes in $X_{\text{error-inducing}}$, 
if we backtrack \emph{only once} for a small backtrack stride $\backtrackStride$, 
and continue the autoregressive sampling process, 
the error rate can be significantly reduced. 

\begin{table}[h]
% \vskip 0.15in
\begin{center}
\begin{small}
% \begin{sc}
\begin{tabular}{ lcc }
\toprule
\textbf{generation configuration} & \textbf{accuracy} \\
\hline
baseline: nucleus sampling top\_p = 0.9 & 0.331 \\
\hline
baseline: greedy argmax sampling & 0.334 \\
\hline
$\backtrackStride = 1$, then nucleus sampling top\_p = 0.9 & 0.366 \\
\hline
$\backtrackStride = 2$, then nucleus sampling top\_p = 0.9 & 0.438 \\
\hline
$\backtrackStride = 4$, then nucleus sampling top\_p = 0.9 & 0.591 \\
\hline
$\backtrackStride = 8$, then nucleus sampling top\_p = 0.9 & 0.790 \\
\bottomrule
\end{tabular}
% \end{sc}
\end{small}
\end{center}
\caption{
At error-inducing prefixes,
a larger backtrack stride $\backtrackStride$ significantly improves completion accuracy (\Cref{def:acc}).
}
\label{table:dyck_backtrack_at_error_inducing}
\end{table}

\subsubsection{Verifier effectively reduces errors}
\label{sec:experiments:synthetic:verifier_reduces_errors}

In \Cref{sec:experiments:synthetic:preventing_errors},
the sampling process forced a backtracking at error-inducing prefixes $X_{\text{error-inducing}}$.
Can the error reduction effect be retained by a \emph{trained} lightweight single-layer verifier $\verifier$ in \Cref{sec:experiments:synthetic:verifier}?
\Cref{table:dyck_verifier_error_inducing} shows that
\algoName (\Cref{alg:sampling_with_backtracking})
using the trained verifier is remarkably effective.
Moreover, in \Cref{sec:experiments:synthetic:predicted_backtracks_were_necessary},
we verify that the predicted backtracks were necessary.

\begin{table}[h]
% \vskip 0.15in
\begin{center}
\begin{small}
%\begin{sc}
\begin{tabular}{ lccc }
\toprule
\textbf{$\backtrackQuota$} & \textbf{$\backtrackStride$} & \textbf{accuracy} \\
\hline
1 & 2 & 0.421 \\
\hline
  & 4 & 0.500 \\
\hline
  & 6 & 0.604 \\
\hline
2 & 2 & 0.457 \\
\hline
  & 4 & 0.634 \\
\hline
  & 6 & 0.762 \\
\hline
4 & 2 & 0.518 \\
\hline
  & 4 & 0.762 \\
\hline
  & 6 & 0.921 \\
\hline
\multicolumn{2}{c}{baseline: nucleus sampling top\_p = 0.9} &  0.331  \\ 
\hline
\multicolumn{2}{c}{baseline: greedy argmax sampling} & 0.334 \\
\bottomrule
\end{tabular}
%\end{sc}
\end{small}
\end{center}sampling
\caption{
When the prompts are error-inducing prefixes,
a single-layer trained verifier significantly improves completion accuracy
using \algoName (\Cref{alg:sampling_with_backtracking}).
A larger backtrack quota $\backtrackQuota$
and a larger backtrack stride $\backtrackStride$ 
are both helpful.
}
\label{table:dyck_verifier_error_inducing}
\end{table}

\subsubsection{\algoName reduces completion errors on unseen OOD prefixes}
\label{sec:experiments:synthetic:verifier_reduces_errors_unseen_ood}

\Cref{table:dyck_verifier_error_inducing} in \Cref{sec:experiments:synthetic:verifier_reduces_errors} reported a significant improvement of accuracy by \algoName (\Cref{alg:sampling_with_backtracking}) 
when the prompts are $X_{\text{error-inducing}}$,
for which the language model $\lm$ made mistakes during completion.
Is the verifier $\verifier$ overfitted to these type of error-inducing prompts?
Can the accuracy improvement generalize to (average-case) out-of-distribution (OOD) prefixes,
i.e. independently sampled strings of the same distribution as $\dyck_{OOD}$ (\Cref{sec:experiments:synthetic:dyck})?

We independently sampled 10000 such out-of-distribution prompts $\dyck_{OOD}^{unseen}$,
and benchmark the accuracy of \algoName (\Cref{alg:sampling_with_backtracking})
against the baselines of 
nucleus sampling top\_p = 0.9 \citep{holtzman2020the}
and standard autoregressive sampling (equivalent to top\_p = 1.0).
\Cref{table:verifier_reduces_errors_unseen_ood}
% (\Cref{sec:appendix:experiments:synthetic:verifier_reduces_errors_unseen_ood})
shows that 
\algoName (\Cref{alg:sampling_with_backtracking}) significantly reduces completion errors.
Crucially, the improvement does not diminish on top of the commonly used baseline of truncating the tail probabilities during sequence sampling.
This verifies the desirable property that 
\algoName 
can be applied in combination with such baselines to further improve accuracy.
We also verify that the accuracy improvement does not hurt diversity (\Cref{sec:experiments:synthetic:diversity}).

Finally, provided with the verifier,
why does the model still make mistakes?
We include additional error analysis in \Cref{sec:experiments:synthetic:error_analysis}.

% TODO: moved table (table:verifier_reduces_errors_unseen_ood) to appendix, can move back
\begin{table}[h]
% \vskip 0.15in
\begin{center}
\begin{small}
%\begin{sc}
\begin{tabular}{ lcccc }
\toprule
\textbf{nucleus sampling top\_p} & \textbf{$\backtrackQuota$} & \textbf{$\backtrackStride$} & \textbf{\#errors $\pm$ std err} \\
% \hline
% 0.0 & 0 & 0 & 178.0 $\pm$ 0.000 \\
% \hline
%   & 4 & 4 & 178 $\pm$ 0.000 \\
\hline
0.9 & 0 & 0 & 240.0 $\pm$ 5.177 \\
\hline
  & 4 & 4 & 179.4 $\pm$ 1.020 \\
\hline
1.0 & 0 & 0 & 461.8 $\pm$ 8.304 \\
\hline
  & 4 & 4 & 200.0 $\pm$ 3.225 \\
\bottomrule
\end{tabular}
%\end{sc}
\end{small}
\end{center}
\caption{
\algoName (\Cref{alg:sampling_with_backtracking}) reduces completion errors on unseen out-of-distribution (OOD) prefixes.
Crucially, the improvement does not diminish on top of commonly used baselines, including
nucleus sampling top\_p = 0.9.
For each setting of top\_p,
we compare \algoName (\Cref{alg:sampling_with_backtracking})
(using backtrack quota $\backtrackQuota = 4$
and backtrack stride $\backtrackStride = 4$)
with the baseline 
(using backtrack quota $\backtrackQuota = 0$
and backtrack stride $\backtrackStride = 0$).
We report the number of completion errors that occur when completing an unseen set of 10000 independently sampled out-of-distribution prompts $\dyck_{OOD}^{unseen}$.
The experiment was repeated 5 times,
and we report the standard errors.
}
\label{table:verifier_reduces_errors_unseen_ood}
\end{table}

\subsection{Generating test cases with pretrained CodeLlama}
\label{sec:experiments:codellama}

Motivated by our findings in \Cref{sec:experiments:synthetic},
we apply essentially the same recipe of 
\algoName (\Cref{alg:sampling_with_backtracking})
to a real-data use case,
and show that \Cref{alg:sampling_with_backtracking} clearly improves the quality vs. query complexity trade-off on top of commonly used baselines, such as 
nucleus sampling \citep{holtzman2020the},
temperature scaling,
best-of-n rejection sampling,
and block best-of-n with process reward model.

\subsubsection{Task setup}
\label{sec:experiments:codellama:task}

A natural practical constrained generation task that requires both accuracy and diversity
is generating test cases for a target function specified by the prompt.
To have an unambiguous notion of groundtruth regarding accuracy and diversity,
we control the target function to be a simple implementation of the \texttt{append} function for Python lists.
Under this setting,
we wrote an evaluator script which analyzes model generated completions, 
measuring the accuracy by checking whether a test case correctly tests list \texttt{append},
and measuring the diversity by checking how many distinct test cases are generated.~\footnote{
Two test cases are different if and only if they test different lists or different appended items.
}

We write a program to systematically generate task prompts,
randomizing over function names and demonstration examples.
Each prompt includes 1 demonstration example specifying the intended output format,
followed by a target function (implementing \texttt{append}),
and finally requests 8 test cases be generated.
Two examples of the prompt are provided in \Cref{table:prompt_example},
and correspondingly, two examples of model completions of these prompts are provided in \Cref{table:generation_example}
in \Cref{sec:experiments:codellama:examples}.

\paragraph{Evaluation metrics}
The test prompts include 10 different target function names that are unseen during training.
Each target function name is independently tested 10 times.
Since each prompt requests 8 test cases,
the total number of test cases requested for each run of a decoding algorithm is 
$8 \times 10 \times 10 = 800$.
We will measure the following metrics:
\begin{enumerate}
    \item $N_\text{distinct correct}$: the number of \textbf{distinct correct} test cases generated. This metric naturally incorporates both accuracy and diversity.
    \item $\text{Acc}_\text{distinct} \coloneqq N_\text{distinct correct} / 800$.
    \item $\mathcal{C}$: the query complexity (analogous to \Cref{def:oracle_complexity}). We measure the total number of queries made to the generator $\lm$ when it completes the prompts.
    Each completion allows at most 384 tokens to be generated, so the max $\mathcal{C}$ is $384 \times 10 \times 10 = 38400$ 
    unless ``block best-of-n" (\Cref{sec:experiments:codellama:verifier}) is used.
\end{enumerate}

We use a pretrained CodeLlama \citep{roziere2023code}
as the generator language model $\lm$, 
which we freeze during our experiments.
We discuss common baselines in \Cref{sec:experiments:codellama:baselines}.
We follow almost the same approach as \Cref{sec:experiments:synthetic:verifier} to train our verifier on this coding task.
We next present technical details and ablation experiments regarding design choices of verifier training in \Cref{sec:experiments:codellama:verifier}.

\subsubsection{Training the verifier}
\label{sec:experiments:codellama:verifier}

We follow almost the same training approach as \Cref{sec:experiments:synthetic:verifier}.
The differences are described below.
The generator language model $\lm$ is a pretrained CodeLlama \citep{roziere2023code} (7B parameters), which we freeze during our experiments.

\paragraph{An intermediate layer provides more informative representations for verifier training than the last layer.}
Instead of training the verifier $\verifier$ on top of the last layer (i.e. layer 31) representations of $\lm$,
we instead treat the layer index as a hyperparameter, 
and conducted a grid search over layer index $\in \{ 3, 7, 11, 15, 19, 23, 27, 31\}$.
Among these candidates, layer 27 representations resulted in the best accuracy.
We therefore exclusively used layer 27 representations in subsequent experiments,
and finally conducted an ablation study on the top-performing setting of the baseline to back-test the impact of using other layers.
\Cref{table:layer27_better_than_layer31} shows that layer 27 outperforms layer 31.
We conjecture that the layer 31 representations may be too specific for the next-token prediction task,
which is not necessarily the optimal for discriminating correct prefixes vs. incorrect ones.~\footnote{
This is in line with some prior works that also observed that 
the final layers of language models tend to be more task-specific than the intermediate layers \citep{liu2019linguistic, kovaleva2019revealing, rogers2021primer}.
}
We also include results for a few other layers near the final layer.
Note that even with a sub-optimally chosen layer,
the accuracy of \algoName (\Cref{alg:sampling_with_backtracking}) 
still outperforms the top-performing settings of the baseline found through grid search.~\footnote{
See \Cref{sec:experiments:codellama:baselines} for details about baselines.
}

\begin{table*}[t]
% \vskip 0.15in
\begin{center}
\begin{small}
%\begin{sc}
\begin{tabular}{ lcc }
\toprule
\textbf{layer index} & \textbf{$\text{Acc}_\text{distinct} \pm$ std err} \\
\hline
27 & 0.714 $\pm$ 0.011 \\
\hline
28 & 0.711 $\pm$ 0.016 \\
\hline
26 & 0.708 $\pm$ 0.018 \\
\hline
30 & 0.706 $\pm$ 0.036 \\
\hline
24 & 0.701 $\pm$ 0.033 \\
\hline
31 & 0.688 $\pm$ 0.028 \\
\hline
29 & 0.676 $\pm$ 0.021 \\
\hline
25 & 0.672 $\pm$ 0.030 \\
\hline
23 & 0.709 $\pm$ 0.017 \\
\hline
3 & 0.700 $\pm$ 0.028 \\
\hline
15 & 0.700 $\pm$ 0.028 \\
\hline
19 & 0.692 $\pm$ 0.028 \\
\hline
7 & 0.691 $\pm$ 0.031 \\
\hline
11 & 0.650 $\pm$ 0.041 \\
\hline
ablation: random verifier & 0.663 $\pm$ 0.027 \\
\hline
baseline: nucleus sampling + temperature scaling & 0.660 $\pm$ 0.042 \\
\bottomrule
\end{tabular}
%\end{sc}
\end{small}
\end{center}
\caption{
Ablation: 
layer 27 representations of CodeLlama outperform layer 31 (the last layer)
in terms of the quality of the error predictor trained based on these features.
We control all other setting to be the same as the top-performing settings of the baseline
(nucleus sampling top\_p = 0.95 \citep{holtzman2020the} and temperature 1.0),
whose performance is also included in the table.
The other rows in this table (layer 27 and layer 31)
refer to applying \algoName (\Cref{alg:sampling_with_backtracking}) 
using backtrack quota $\backtrackQuota = 4$,
backtrack stride $\backtrackStride = 4$,
and verifiers trained on layers 24, ..., 31 of the generator (CodeLlama), respectively.
The row \emph{ablation: random verifier} refers to 
a verifier that returns Uniform[0, 1], and uses the same $\backtrackQuota$, $\backtrackStride$ as the above.
The experiment was repeated 5 times,
and we report the standard errors.
The rows are sorted by mean $\text{Acc}_\text{distinct}$ (\Cref{sec:experiments:codellama:task}).
}
\label{table:layer27_better_than_layer31}
\end{table*}

\paragraph{With limited backtrack quota, it is better to be conservative in its usage.}
The verifier $\verifier$ is trained with binary labels
(1 if correct, 0 if wrong).
Although there are a roughly equal number of training samples whose labels are 0 or are 1,
using 0.5 as the error prediction threshold turned out to be suboptimal.
Since our \algoName (\Cref{alg:sampling_with_backtracking}) 
only allows a small backtrack quota $\backtrackQuota = 4$,
it makes sense to only use backtrack quota when the error predictor is very confident that the current intermediate generation is wrong.
Moreover, compared with our synthetic Dyck grammar setting (target length = 32) (\Cref{sec:experiments:synthetic}),
our code generation setting allows much longer generations (up to 384),
which further justifies conservatively spending the small backtrack quota $\backtrackQuota$.
Consequently, we consider decreasing the error prediction threshold to 0.1.
\Cref{table:codellama_error_prediction_threshold} shows that 0.1 is a better error prediction threshold than the default 0.5 in all settings we tried.

\begin{table*}[t]
% \vskip 0.15in
\begin{center}
\begin{small}
%\begin{sc}
\begin{tabular}{ lcccccc }
\toprule
\textbf{$\backtrackQuota$} & \textbf{$\backtrackStride$} & \textbf{top\_p} & \textbf{temperature} & \textbf{error prediction threshold} & \textbf{$\text{Acc}_\text{distinct}$ $\pm$ std err} \\
\hline
\hline
4 & 4 & 0.95 & 1.0 & 0.1 &  0.714 $\pm$ 0.011 \\
\hline
4 & 4 & 0.95 & 1.0 & 0.5 &  0.676 $\pm$ 0.019 \\
\hline
\hline
4 & 4 & 1.0 & 1.0 & 0.1 &  0.639 $\pm$ 0.061 \\
\hline
4 & 4 & 1.0 & 1.0 & 0.5 &  0.604 $\pm$ 0.047 \\
\hline
\hline
4 & 4 & 1.0 & 1.2 & 0.1 &  0.440 $\pm$ 0.026 \\
\hline
4 & 4 & 1.0 & 1.2 & 0.5 &  0.334 $\pm$ 0.013 \\
\hline
\hline
4 & 10 & 1.0 & 1.0 & 0.1 &  0.622 $\pm$ 0.046 \\
\hline
4 & 10 & 1.0 & 1.0 & 0.1 &  0.604 $\pm$ 0.030 \\
\bottomrule
\end{tabular}
%\end{sc}
\end{small}
\end{center}
\caption{
Ablation: 0.1 is a better error prediction threshold than the default 0.5 in all settings we tried,
including various
nucleus sampling \citep{holtzman2020the} top\_p,
temperature scaling,
and backtrack stride $\backtrackStride$.
In this table, we divide the rows into groups of 2,
separated by double horizontal lines,
such that within each group,
the only difference is the error prediction threshold.
In all groups, 
0.1 leads to higher $\text{Acc}_\text{distinct}$ than 0.5.
The experiment was repeated 5 times,
and we report the standard errors.
}
\label{table:codellama_error_prediction_threshold}
\end{table*}

\paragraph{Block verifier.}
Our verifier applies to the token level,
i.e. predicting an accept/reject action after the generator $\lm$ generates each token.
In many practical settings (including ours),
it is natural to divide the generated output into \emph{blocks} 
(each block may contain multiple tokens),
e.g. in writing math proofs, each block may correspond to one reasoning step;
in writing codes, each block may correspond to one line of codes.
Recent works achieved strong empirical performance by 
generating multiple candidates for each block of intermediate model generations,
train process reward models that evaluate each candidate,
and select the best-scoring candidate
(see e.g. \citet{wu2024inference} and references therein).
We refer to this as the ``block-best-of-n" approach.
To compare with such ``block-best-of-n" baselines,
we train ``block verifiers" $\verifier_{\text{block}}$ which scores prefixes that are full lines of model output for our task.
We will show that this ``block best-of-n" approach is helpful, 
but is outperformed by our \algoName (\Cref{alg:sampling_with_backtracking})
in terms of accuracy-efficiency trade-off.

\paragraph{Does a deeper verifier perform better?}

The above experiments follow \Cref{sec:experiments:synthetic:verifier} in training a single-linear-layer verifier.
In this section,
we test the effects of scaling up the verifier depth.
Specifically, we test verifiers based on Multi-Layer Perceptrons \citep{rosenblatt1958perceptron} of depths 2, 4, 8,
with ReLU activations \citep{nair2010rectified} between adjacent parameterized layers.
\Cref{table:verifier_more_mlp_layers} shows that more MLP layers did not outperform the 1-linear-layer verifier
even though they can be trained to similar \emph{error-predicting} accuracies,
measured by their accuracy in predicting whether a prefix is correct or incorrect
on a held-old validation set of prompts for our task (\Cref{sec:experiments:codellama:task})
followed by partial generations by CodeLlama.
In other sections of this paper, unless otherwise noted, we always use a single-linear-layer verifier for \algoName (\Cref{alg:sampling_with_backtracking})
(and of course, no verifier for baselines (\Cref{sec:experiments:codellama:baselines}) ).

\begin{table*}[t]
% \vskip 0.15in
\begin{center}
\begin{small}
%\begin{sc}
\begin{tabular}{ lccc }
\toprule
\textbf{verifier \# MLP layers} & \textbf{verifier validation accuracy} & \textbf{$\text{Acc}_\text{distinct} \pm$ std err} \\
\hline
1 & 0.96 & 0.714 $\pm$ 0.011 \\
\hline
4 & 0.97 & 0.699 $\pm$ 0.038 \\
\hline
2 & 0.97 & 0.687 $\pm$ 0.035 \\
\hline
8 & 0.97 & 0.684 $\pm$ 0.015 \\
\hline
ablation: random verifier &  0.50 & 0.663 $\pm$ 0.027 \\
\hline
baseline: nucleus sampling + temperature scaling & N/A & 0.660 $\pm$ 0.042 \\
\bottomrule
\end{tabular}
%\end{sc}
\end{small}
\end{center}
\caption{
Ablation: 
Deeper verifiers do not outperform the 1-linear-layer verifier
even though they can be trained to similar \emph{error-predicting} accuracies on held-old validation set.
We control all other setting to be the same as the top-performing settings of the baseline
(nucleus sampling top\_p = 0.95 \citep{holtzman2020the} and temperature 1.0),
whose performance is also included in the table.
The other rows in this table
refer to applying \algoName (\Cref{alg:sampling_with_backtracking}) 
using backtrack quota $\backtrackQuota = 4$,
backtrack stride $\backtrackStride = 4$,
and verifiers with 1, 2, 4, 8 layers, respectively.
The row \emph{ablation: random verifier} refers to 
a verifier that returns Uniform[0, 1], and uses the same $\backtrackQuota$, $\backtrackStride$ as the above.
The experiment was repeated 5 times,
and we report the standard errors.
The rows are sorted by mean $\text{Acc}_\text{distinct}$ (\Cref{sec:experiments:codellama:task}).
}
\label{table:verifier_more_mlp_layers}
\end{table*}

\paragraph{Where are the potentials for further improving $\text{Acc}_\text{distinct}$?}
How optimal are our verifiers,
and what are some ways to further improve them?
To probe these potentials,
we wrote a rule-based groundtruth verifier for our task (\Cref{sec:experiments:codellama:task})
and used it as a drop-in replacement of our trained verifier.
\Cref{table:codellama_groundtruth} shows that 
the $\text{Acc}_\text{distinct}$ enabled by our trained verifier almost reached the $\text{Acc}_\text{distinct}$ enabled by the groundtruth verifier,
showing that improving verifier training may not be the most fruitful direction for further improvement.
Interestingly, using a much larger $\backtrackQuota$ or $\backtrackStride$ (increasing from 4 to 10) 
does not necessarily improve the accuracy (sometimes even \emph{decreasing} the accuracy).
We conjecture that in these experiments, 
the (imperfect) generator oracle (CodeLlama),
not the verifier,
was the bottleneck for $\text{Acc}_\text{distinct}$.
As a result,
unnecessarily backtracking and forcing the model to re-generate more tokens may increase the chance that the model makes mistakes.

\begin{table*}[h]
% \vskip 0.15in
\begin{center}
\begin{small}
%\begin{sc}
\begin{tabular}{ lcccc }
\toprule
\textbf{verifier type} & \textbf{$\backtrackQuota$} & \textbf{$\backtrackStride$} & \textbf{$\text{Acc}_\text{distinct}$ $\pm$ std err} \\
\hline
\hline
groundtruth & 4 & 4 & 0.719 $\pm$ 0.022 \\
\hline
groundtruth & 10 & 4 & 0.717 $\pm$ 0.015 \\
\hline
\hline
trained & 4 & 4 & 0.714 $\pm$ 0.011 \\
\hline
trained & 10 & 4 & 0.692 $\pm$ 0.025 \\
\hline
\hline
ablation: random verifier & 4 & 4  & 0.663 $\pm$ 0.027 \\
\hline
baseline: nucleus sampling + temperature scaling & 0 & 0  & 0.660 $\pm$ 0.042 \\
\hline
\hline
trained & 4 & 10 & 0.622 $\pm$ 0.046 \\
\bottomrule
\end{tabular}
%\end{sc}
\end{small}
\end{center}
\caption{
Ablation: 
Our trained verifier approaches the accuracy of the groundtruth verifier,
evaluated by their ability to assist CodeLlama in completing our test case generation task (\Cref{sec:experiments:codellama:task})
using \algoName (\Cref{alg:sampling_with_backtracking}).
In these experiments, we control the
nucleus sampling \citep{holtzman2020the} top\_p = 0.95 and temperature scaling = 1.0
which are the optimal setting for baseline, found by grid search (\Cref{sec:experiments:codellama:baselines}).
The rows are sorted by $\text{Acc}_\text{distinct}$.
The row \emph{ablation: random verifier} refers to 
a verifier that returns Uniform[0, 1].
Interestingly, using a much larger $\backtrackQuota$ or $\backtrackStride$ does not necessarily improve the accuracy (sometimes even \emph{decreasing} the accuracy).
We conjecture that the generator model, CodeLlama, is imperfect, 
so unnecessarily backtracking and forcing the model to re-generate more tokens may increase the chance that the model makes mistakes.
The experiment was repeated 5 times,
and we report the standard errors.
}
\label{table:codellama_groundtruth}
\end{table*}

\subsubsection{\algoName improves accuracy}
\label{sec:experiments:codellama:verifier_improves_accuracy}

In this section, we show that 
\algoName (\Cref{alg:sampling_with_backtracking})
achieves higher $\text{Acc}_\text{distinct}$
than all the baselines described in \Cref{sec:experiments:codellama:baselines}.
Similar to our observations based on the synthetic Dyck grammar data (\Cref{sec:experiments:synthetic:verifier_reduces_errors_unseen_ood}), 
the improvement does not diminish on top of commonly used baselines.
This verifies the desirable property that 
\algoName 
(\Cref{alg:sampling_with_backtracking})
can be applied in combination with commonly used baselines to further improve accuracy.
The primary comparisons are reported in \Cref{table:codellama_verifier_improves_accuracy_simplified}, 
% (\Cref{sec:appendix:experiments:codellama:verifier_improves_accuracy}),
and additional results are in
\Cref{table:codellama_verifier_improves_accuracy}
in \Cref{sec:experiments:codellama:full_results}. 
Moreover, in \Cref{sec:experiments:codellama:ood},
we show that 
analogous to our observations on the synthetic Dyck grammar (\Cref{sec:experiments:synthetic:verifier_reduces_errors_unseen_ood}),
\algoName (\Cref{alg:sampling_with_backtracking})
generalizes better to \emph{out-of-distribution} prompts than baselines.

% TODO: moved the table (table:codellama_verifier_improves_accuracy_simplified) to appendix, can move back later
\begin{table}[t!]
% \vskip 0.15in
\begin{center}
\begin{small}
%\begin{sc}
\begin{tabular}{ lccccc }
\toprule
\textbf{$\backtrackQuota$} & \textbf{$\backtrackStride$} & \textbf{top\_p} & \textbf{\temperature} & \textbf{block BoN} & \textbf{$\text{Acc}_\text{distinct}$ $\pm$ std err} \\
\hline
4 & 4 & 0.95 & 1.0 &  &  0.714 $\pm$ 0.011 \\
\hline
0 &  & 0.95 & 1.0 & 2 &  0.684 $\pm$ 0.038 \\
\hline
0 &  & 0.95 & 1.0 &  &  0.660 $\pm$ 0.042 \\
\hline
0 &  & 0.95 & 1.0 & 4 &  0.623 $\pm$ 0.036 \\
\hline
0 &  & 0.95 & 1.0 & 8 &  0.559 $\pm$ 0.038 \\
\hline
\hline
4 & 4 & 1.0 & 1.0 &  &  0.639 $\pm$ 0.061 \\
\hline
4 & 10 & 1.0 & 1.0 &  &  0.622 $\pm$ 0.046 \\
\hline
0 &  & 1.0 & 1.0 &  &  0.504 $\pm$ 0.025 \\
\hline
\hline
4 & 4 & 1.0 & 1.2 &  &  0.440 $\pm$ 0.026 \\
\hline
0 &  & 1.0 & 1.2 &  &  0.269 $\pm$ 0.025 \\
\hline
\hline
0 &  & 0.0 & 1.0 &  &  0.013 $\pm$ 0.000 \\
\bottomrule
\end{tabular}
%\end{sc}
\end{small}
\end{center}
\vspace{-5mm}
\caption{
\algoName (\Cref{alg:sampling_with_backtracking}) improves accuracy
and outperforms 
nucleus sampling top\_p,
temperature scaling \temperature,
and block best-of-n (BoN) (\Cref{sec:experiments:codellama:verifier}).
In this table, we divide the rows into groups,
separated by double horizontal lines,
such that each group uses the same top\_p and temperature.
The backtrack quota $\backtrackQuota = 0$ means a baseline algorithm that does not use the verifier.
$\backtrackQuota > 0$ means \algoName 
with the corresponding $\backtrackQuota$ and $\backtrackStride$.
\emph{block BoN} specifies the number of candidates generated for each block;
empty block BoN means not using block best-of-n.
In all groups, 
\algoName
leads to higher $\text{Acc}_\text{distinct}$ 
than all other methods.
The last group corresponds to
argmax greedy decoding,
which has low $\text{Acc}_\text{distinct}$ due to low diversity.
The experiment was repeated 5 times,
and we report the standard errors.
The complete set of experiments are reported in a larger \Cref{table:codellama_verifier_improves_accuracy} in \Cref{sec:experiments:codellama:full_results}.
}
\label{table:codellama_verifier_improves_accuracy_simplified}
\end{table}

\clearpage
\subsubsection{\algoName is query efficient}
\label{sec:experiments:codellama:efficient}

In this section, we show that 
\algoName (\Cref{alg:sampling_with_backtracking})
achieves a better tradeoff between $\text{Acc}_\text{distinct}$
and
query efficiency $\mathcal{C}$
than all the baselines described in \Cref{sec:experiments:codellama:baselines}.
The primary comparisons are visualized in \Cref{fig:codellama_query_efficiency} (in this section)~\footnote{
    This visualization in \Cref{fig:codellama_query_efficiency} slightly favors the ``block best-of-n sampling" baseline, 
    because its implementation stops the decoding process once the requested number of test cases are generated, 
    whereas when running our algorithm or non-best-of-n baselines, 
    the model is allowed to (and in fact does indeed) generate irrelevant tokens afterwards, 
    which hurts query complexity.
    Even under this disadvantage, 
    \algoName still outperforms the ``block best-of-n sampling" baselines.
}
and \Cref{fig:codellama_query_efficiency_no_bon}
(in \Cref{sec:experiments:codellama:visualize_efficiency}).
Numerical values of $\mathcal{C}$ are reported in
\Cref{table:codellama_verifier_improves_accuracy}
in \Cref{sec:experiments:codellama:full_results}. 

% \begin{remark}
    
% \end{remark}

\begin{figure*}[h!]
  \centering
  \begin{minipage}[b]{1.0\textwidth}  % 1.0 for arxiv
    \centering
    \includegraphics[width=1.0\textwidth]{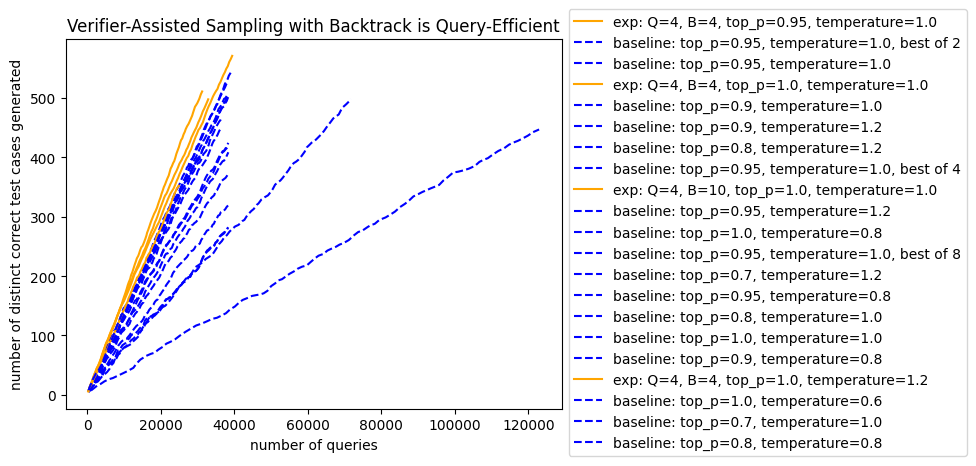}  % 0.5 for arxiv
  \end{minipage}
  \vspace{-7mm}
  \caption{
  \algoName (\Cref{alg:sampling_with_backtracking}) is query-efficient.
  The horizontal axis denotes query complexity $\mathcal{C}$,
  and the vertical axis denotes the number of distinct correct test cases generated $N_\text{distinct correct}$,
  both defined in \Cref{sec:experiments:codellama:task}.
  Blue dashed lines correspond to the baselines (described in \Cref{sec:experiments:codellama:baselines}), 
  whereas orange solid lines correspond to \algoName with various $\backtrackQuota$ and $\backtrackStride$, 
  both defined in \Cref{alg:sampling_with_backtracking}.
  Since the slopes of the orange curves are visibly greater than the slopes of the blue curves,
  we conclude that \algoName is more query-efficient than baselines.
  The experiment was repeated 5 times,
  and each dot is the average metric of these 5 runs.
  The specific numbers and standard errors are reported in \Cref{table:codellama_verifier_improves_accuracy}.
  A more zoomed-in version of this plot is in \Cref{fig:codellama_query_efficiency_no_bon}.
  }  
  \label{fig:codellama_query_efficiency}
  % \vspace{-0.3cm}
\end{figure*}

% \yuchen{Additional results are included in \Cref{sec:appendix:experiments}.}

\section{Related work}
\label{sec:related_works}

\paragraph{Inference-time scaling for language models}
Practical language generation tasks typically impose various task-specific constraints in addition to the general grammatical rules of language.
One effective way to improve the chance of satisfying such constraints
is to increase the inference-time compute 
through search and/or rejection sampling.
There has been a long history of prior works that employ inference-time scaling in the language generation context, 
dating as far back as beam search \citep{lowerre1976harpy,hayes1976speech,ow1988filtered,jurafsky2000speech,graves2012sequence}.
Much more recently, as researchers develop the techniques for language models to follow instructions 
(see the survey by \citet{zhang2023instruction} and references therein),
more creative designs for inference-time scaling algorithms have become viable 
\citep{wang2022self,yao2023tree,zhang2023planning,zhou2023language,choi2023kcts,liu2024don,xie2024self,snell2024scaling,zhao2024probabilistic},
and see \citet{wu2024inference} for a recent survey on cost-performance tradeoffs of these approaches.
Compared to these approaches in the literature, 
our \algoName (\Cref{alg:sampling_with_backtracking}) shares some features with lookahead search \citep{snell2024scaling}
(specifically, the rejection decision at the current position is based on the verifier decision at some future position).
However, two main differences are:
(1) \algoName (\Cref{alg:sampling_with_backtracking})
does not use a beam (i.e. does not need to generate multiple candidates, thus reducing query complexity),
and (2) \Cref{alg:sampling_with_backtracking} uses a different sampling approach (namely argmax) for the backtracked positions (we verify in \Cref{sec:appendix:experiments:algo} that in some settings this significantly improves the accuracy).
It is a natural future research direction to design inference algorithms that combine the advantages of the two.

\paragraph{Incorporating a process reward model to assist language generation}
Among the vast design space for inference-time scaling,
process reward modeling has been proven to be an important component common to many LLM systems
\citep{polu2020generative,uesato2022solving,ma2023let,lightman2023let,wang2024math}.
The process verifier which we study (\Cref{def:verifier}) is a special case of such process reward model if we restrict the output to be binary.
However, there are still challenging open problems around process reward modeling,
such as how to properly define the ``blocks"
\citep{guo2025deepseek}
(see also our definitions in the ``Block verifier" part of \Cref{sec:experiments:codellama:verifier}).
Towards bringing more clarity to these open questions,
our work develops a theoretical framework for reasoning about the query complexity of process verifiers.
Moreover, our experiments suggest the potentials of a lightweight process verifier in improving the query complexity, accuracy, and diversity of constrained generation.
In particular, our theory and experiments suggest
(1) the ``blocks" do not necessarily have to be carefully designed --- setting each token as a block might potentially suffice, at least in some more structured domains such as codes;
(2) \emph{backtracking} (\Cref{alg:sampling_with_backtracking}, \Cref{sec:experiments}) is a robustly effective strategy that should be applied in combination with process verifiers.
A possible extension of the type of verifier we study (\Cref{def:verifier}) is: 
instead of outputting binary acceptance / rejection decisions, the verifier could return a probability of accepting each prefix \cite{yang2021fudge}.
However, some tasks may require that the output distribution should match some target distribution,
and it may be challenging to ensure that the acceptance probability is well-calibrated in order to satisfy this requirement.
% We discuss additional related works in Appendix~\ref{sec:appendix:related_works}.

\paragraph{Controlled synthetic data distribution as a sandbox for studying language models}
Our Dyck grammar distribution most closely follows \citet{wen2023uninterpretability} 
(though we switched to a fixed-sequence-length setting, and used unbalanced bracket type probability, instead of length extrapolation, to define the criteria for a prompt to be \emph{out-of-distribution}).
Dyck grammar was also used in other prior works \citep{hewitt2020rnns, ebrahimi2020self, yao2021self, liu2023same, liu2023Transformers} to study language models.
Dyck grammar can be seen as a special case of the task (specifically context-free grammar) considered in SynCode \citep{ugare2024syncode}.
Other synthetic data distributions have been used to study various aspects of language models in prior works, including
representational capability \citep{bhattamishra2020ability, li2021limitations, zhang2022unveiling, zhao2023Transformers}, 
statistical sample complexity \citep{edelman2022inductive},
optimization process \citep{lu2021on, jelassi2022vision,li2023Transformers,bietti2023birth},
sampling \citep{li2024promises},
and architectural limitations \citep{liu2023exposing},
and see references cited therein.

% \vspace{-3mm}
\section{Conclusion}

We introduce a new theoretical framework for elucidating the design space of verifiers and correspondingly a simple family of rejection-sampling-based inference algorithms.
In particular, our theory proves the computational benefits of incorporating a \emph{process verifier},
measured by the \emph{query complexity} of calling the generator.
On the other hand, our theory also reveals the subtleties:
straightforwardly applying a process verifier in a Tokenwise rejection sampling algorithm may unintentionally re-weigh the distribution among sequences that satisfy the constraints,
which could be undesirable for settings that require a strong notion of distributional \emph{calibration}.
Empirically, through fine-grained experiments on both synthetic and realistic data,
we show that the Tokenwise rejection sampling algorithm, when combined with \emph{backtracking},
is a robustly effective recipe for reducing query
complexity, improving accuracy, and maintaining diversity. 
For future works, we hope the theoretical framework and empirical observations can inspire systematic characterization of the strengths and weaknesses of the diverse set of rejection-sampling-based inference-time algorithms.
Concrete open problems at the intersection of theory and experiments include 
investigating the realistic and necessary conditions on the verifiers for the inference-time algorithm to
achieve distributional calibration
(e.g. it is unrealistic in some language generation setting to assume that a verifier returns the calibrated acceptance probability in rejection sampling),
and synergistically designing query-efficient verifier-assisted generation algorithms.

\subsubsection*{ACKNOWLEDGEMENTS}
We thank Bingbin Liu for insightful discussions.

AR and YL are supported in part by NSF awards IIS-2211907, CCF-2238523, IIS-2403275, an Amazon Research Award, a Google Research Scholar Award, an OpenAI Superalignment Fast Grant, and DoD Award N000142512124.

Part of this work was done when YL and AR
were visiting the Simons Institute for the Theory of Computing.
We thank the Simons Institute for hosting us.
We also thank the Simons Foundation and Google for sponsoring computational resources for us and other program visitors,
and we thank Matus Telgarsky and NYU IT for managing these computational resources.

% \yuchen{Other funding to acknowledge?}

\newpage
\bibliography{references}

\begin{thebibliography}{67}
\providecommand{\natexlab}[1]{#1}
\providecommand{\url}[1]{\texttt{#1}}
\expandafter\ifx\csname urlstyle\endcsname\relax
  \providecommand{\doi}[1]{doi: #1}\else
  \providecommand{\doi}{doi: \begingroup \urlstyle{rm}\Url}\fi

\bibitem[Ackley et~al.(1985)Ackley, Hinton, and Sejnowski]{ackley1985learning}
David~H Ackley, Geoffrey~E Hinton, and Terrence~J Sejnowski.
\newblock A learning algorithm for boltzmann machines.
\newblock \emph{Cognitive science}, 9\penalty0 (1):\penalty0 147--169, 1985.

\bibitem[Bhattamishra et~al.(2020)Bhattamishra, Ahuja, and Goyal]{bhattamishra2020ability}
Satwik Bhattamishra, Kabir Ahuja, and Navin Goyal.
\newblock On the {A}bility and {L}imitations of {T}ransformers to {R}ecognize {F}ormal {L}anguages.
\newblock In \emph{Proceedings of the 2020 Conference on Empirical Methods in Natural Language Processing (EMNLP)}, pp.\  7096--7116, Online, November 2020. Association for Computational Linguistics.
\newblock \doi{10.18653/v1/2020.emnlp-main.576}.
\newblock URL \url{https://aclanthology.org/2020.emnlp-main.576}.

\bibitem[Bietti et~al.(2023)Bietti, Cabannes, Bouchacourt, Jegou, and Bottou]{bietti2023birth}
Alberto Bietti, Vivien Cabannes, Diane Bouchacourt, Herve Jegou, and Leon Bottou.
\newblock Birth of a transformer: A memory viewpoint.
\newblock \emph{Advances in Neural Information Processing Systems}, 36, 2023.

\bibitem[Block et~al.(2024)Block, Foster, Krishnamurthy, Simchowitz, and Zhang]{block2024butterfly}
Adam Block, Dylan~J Foster, Akshay Krishnamurthy, Max Simchowitz, and Cyril Zhang.
\newblock Butterfly effects of {SGD} noise: Error amplification in behavior cloning and autoregression.
\newblock In \emph{The Twelfth International Conference on Learning Representations}, 2024.
\newblock URL \url{https://openreview.net/forum?id=CgPs04l9TO}.

\bibitem[Brown et~al.(2020)Brown, Mann, Ryder, Subbiah, Kaplan, Dhariwal, Neelakantan, Shyam, Sastry, Askell, Agarwal, Herbert-Voss, Krueger, Henighan, Child, Ramesh, Ziegler, Wu, Winter, Hesse, Chen, Sigler, Litwin, Gray, Chess, Clark, Berner, McCandlish, Radford, Sutskever, and Amodei]{brown2020language}
Tom Brown, Benjamin Mann, Nick Ryder, Melanie Subbiah, Jared~D Kaplan, Prafulla Dhariwal, Arvind Neelakantan, Pranav Shyam, Girish Sastry, Amanda Askell, Sandhini Agarwal, Ariel Herbert-Voss, Gretchen Krueger, Tom Henighan, Rewon Child, Aditya Ramesh, Daniel Ziegler, Jeffrey Wu, Clemens Winter, Chris Hesse, Mark Chen, Eric Sigler, Mateusz Litwin, Scott Gray, Benjamin Chess, Jack Clark, Christopher Berner, Sam McCandlish, Alec Radford, Ilya Sutskever, and Dario Amodei.
\newblock Language models are few-shot learners.
\newblock In H.~Larochelle, M.~Ranzato, R.~Hadsell, M.F. Balcan, and H.~Lin (eds.), \emph{Advances in Neural Information Processing Systems}, volume~33, pp.\  1877--1901. Curran Associates, Inc., 2020.
\newblock URL \url{https://proceedings.neurips.cc/paper/2020/file/1457c0d6bfcb4967418bfb8ac142f64a-Paper.pdf}.

\bibitem[Browne et~al.(2012)Browne, Powley, Whitehouse, Lucas, Cowling, Rohlfshagen, Tavener, Liebana, Samothrakis, and Colton]{Browne2012ASO}
Cameron Browne, Edward~Jack Powley, Daniel Whitehouse, Simon M.~M. Lucas, Peter~I. Cowling, Philipp Rohlfshagen, Stephen Tavener, Diego~Perez Liebana, Spyridon Samothrakis, and Simon Colton.
\newblock A survey of monte carlo tree search methods.
\newblock \emph{IEEE Transactions on Computational Intelligence and AI in Games}, 4:\penalty0 1--43, 2012.
\newblock URL \url{https://api.semanticscholar.org/CorpusID:9316331}.

\bibitem[Chen et~al.(2022)Chen, Chewi, Li, Li, Salim, and Zhang]{chen2022sampling}
Sitan Chen, Sinho Chewi, Jerry Li, Yuanzhi Li, Adil Salim, and Anru~R Zhang.
\newblock Sampling is as easy as learning the score: theory for diffusion models with minimal data assumptions.
\newblock \emph{arXiv preprint arXiv:2209.11215}, 2022.

\bibitem[Choi et~al.(2023)Choi, Fang, Wang, and Song]{choi2023kcts}
Sehyun Choi, Tianqing Fang, Zhaowei Wang, and Yangqiu Song.
\newblock Kcts: knowledge-constrained tree search decoding with token-level hallucination detection.
\newblock \emph{arXiv preprint arXiv:2310.09044}, 2023.

\bibitem[Cobbe et~al.(2021)Cobbe, Kosaraju, Bavarian, Chen, Jun, Kaiser, Plappert, Tworek, Hilton, Nakano, Hesse, and Schulman]{cobbe2021verifiers}
Karl Cobbe, Vineet Kosaraju, Mohammad Bavarian, Mark Chen, Heewoo Jun, Lukasz Kaiser, Matthias Plappert, Jerry Tworek, Jacob Hilton, Reiichiro Nakano, Christopher Hesse, and John Schulman.
\newblock Training verifiers to solve math word problems, 2021.
\newblock URL \url{https://arxiv.org/abs/2110.14168}.

\bibitem[Dehghani et~al.(2021)Dehghani, Arnab, Beyer, Vaswani, and Tay]{dehghani2021efficiency}
Mostafa Dehghani, Anurag Arnab, Lucas Beyer, Ashish Vaswani, and Yi~Tay.
\newblock The efficiency misnomer.
\newblock \emph{arXiv preprint arXiv:2110.12894}, 2021.

\bibitem[Ebrahimi et~al.(2020)Ebrahimi, Gelda, and Zhang]{ebrahimi2020self}
Javid Ebrahimi, Dhruv Gelda, and Wei Zhang.
\newblock How can self-attention networks recognize {D}yck-n languages?
\newblock In \emph{Findings of the Association for Computational Linguistics: EMNLP 2020}, pp.\  4301--4306, Online, November 2020. Association for Computational Linguistics.
\newblock \doi{10.18653/v1/2020.findings-emnlp.384}.
\newblock URL \url{https://aclanthology.org/2020.findings-emnlp.384}.

\bibitem[Edelman et~al.(2022)Edelman, Goel, Kakade, and Zhang]{edelman2022inductive}
Benjamin~L Edelman, Surbhi Goel, Sham Kakade, and Cyril Zhang.
\newblock Inductive biases and variable creation in self-attention mechanisms.
\newblock In Kamalika Chaudhuri, Stefanie Jegelka, Le~Song, Csaba Szepesvari, Gang Niu, and Sivan Sabato (eds.), \emph{Proceedings of the 39th International Conference on Machine Learning}, volume 162 of \emph{Proceedings of Machine Learning Research}, pp.\  5793--5831. PMLR, 17--23 Jul 2022.
\newblock URL \url{https://proceedings.mlr.press/v162/edelman22a.html}.

\bibitem[Graves(2012)]{graves2012sequence}
Alex Graves.
\newblock Sequence transduction with recurrent neural networks.
\newblock \emph{arXiv preprint arXiv:1211.3711}, 2012.

\bibitem[Guo et~al.(2025)Guo, Yang, Zhang, Song, Zhang, Xu, Zhu, Ma, Wang, Bi, et~al.]{guo2025deepseek}
Daya Guo, Dejian Yang, Haowei Zhang, Junxiao Song, Ruoyu Zhang, Runxin Xu, Qihao Zhu, Shirong Ma, Peiyi Wang, Xiao Bi, et~al.
\newblock Deepseek-r1: Incentivizing reasoning capability in llms via reinforcement learning.
\newblock \emph{arXiv preprint arXiv:2501.12948}, 2025.

\bibitem[Hao et~al.(2023)Hao, Gu, Ma, Hong, Wang, Wang, and Hu]{hao2023reasoning}
Shibo Hao, Yi~Gu, Haodi Ma, Joshua~Jiahua Hong, Zhen Wang, Daisy~Zhe Wang, and Zhiting Hu.
\newblock Reasoning with language model is planning with world model.
\newblock \emph{arXiv preprint arXiv:2305.14992}, 2023.

\bibitem[Hayes-Roth et~al.(1976)Hayes-Roth, Fox, Gill, Mostow, and Reddy]{hayes1976speech}
P~Hayes-Roth, M~Fox, G~Gill, DJ~Mostow, and R~Reddy.
\newblock Speech understanding systems: Summary of results of the five-year research effort, 1976.

\bibitem[Hewitt et~al.(2020)Hewitt, Hahn, Ganguli, Liang, and Manning]{hewitt2020rnns}
John Hewitt, Michael Hahn, Surya Ganguli, Percy Liang, and Christopher~D. Manning.
\newblock {RNN}s can generate bounded hierarchical languages with optimal memory.
\newblock In \emph{Proceedings of the 2020 Conference on Empirical Methods in Natural Language Processing (EMNLP)}, pp.\  1978--2010, Online, November 2020. Association for Computational Linguistics.
\newblock \doi{10.18653/v1/2020.emnlp-main.156}.
\newblock URL \url{https://www.aclweb.org/anthology/2020.emnlp-main.156}.

\bibitem[Holtzman et~al.(2020)Holtzman, Buys, Du, Forbes, and Choi]{holtzman2020the}
Ari Holtzman, Jan Buys, Li~Du, Maxwell Forbes, and Yejin Choi.
\newblock The curious case of neural text degeneration.
\newblock In \emph{International Conference on Learning Representations}, 2020.
\newblock URL \url{https://openreview.net/forum?id=rygGQyrFvH}.

\bibitem[Jelassi et~al.(2022)Jelassi, Sander, and Li]{jelassi2022vision}
Samy Jelassi, Michael~Eli Sander, and Yuanzhi Li.
\newblock Vision transformers provably learn spatial structure.
\newblock In Alice~H. Oh, Alekh Agarwal, Danielle Belgrave, and Kyunghyun Cho (eds.), \emph{Advances in Neural Information Processing Systems}, 2022.
\newblock URL \url{https://openreview.net/forum?id=eMW9AkXaREI}.

\bibitem[Jurafsky \& Martin(2000)Jurafsky and Martin]{jurafsky2000speech}
Daniel Jurafsky and James~H Martin.
\newblock Speech and language processing: An introduction to natural language processing, computational linguistics, and speech recognition, 2000.

\bibitem[Karp(1972)]{reducibilityKarp}
Richard Karp.
\newblock Reducibility among combinatorial problems.
\newblock In \emph{Complexity of Computer Computations}, volume~40, pp.\  85--103, 01 1972.
\newblock ISBN 978-3-540-68274-5.
\newblock \doi{10.1007/978-3-540-68279-0_8}.

\bibitem[Kellerer et~al.(2004)Kellerer, Pferschy, and Pisinger]{knapsackProblems}
Hans Kellerer, Ulrich Pferschy, and David Pisinger.
\newblock \emph{Knapsack Problems}.
\newblock Springer, 01 2004.
\newblock ISBN 978-3-540-40286-2.
\newblock \doi{10.1007/978-3-540-24777-7}.

\bibitem[Kovaleva et~al.(2019)Kovaleva, Romanov, Rogers, and Rumshisky]{kovaleva2019revealing}
Olga Kovaleva, Alexey Romanov, Anna Rogers, and Anna Rumshisky.
\newblock Revealing the dark secrets of {BERT}.
\newblock In \emph{Proceedings of the 2019 Conference on Empirical Methods in Natural Language Processing and the 9th International Joint Conference on Natural Language Processing (EMNLP-IJCNLP)}, pp.\  4365--4374, Hong Kong, China, November 2019. Association for Computational Linguistics.
\newblock \doi{10.18653/v1/D19-1445}.
\newblock URL \url{https://aclanthology.org/D19-1445}.

\bibitem[Li et~al.(2024{\natexlab{a}})Li, Fang, Smyrnis, Ivgi, Jordan, Gadre, Bansal, Guha, Keh, Arora, et~al.]{li2024datacomp}
Jeffrey Li, Alex Fang, Georgios Smyrnis, Maor Ivgi, Matt Jordan, Samir Gadre, Hritik Bansal, Etash Guha, Sedrick Keh, Kushal Arora, et~al.
\newblock Datacomp-lm: In search of the next generation of training sets for language models.
\newblock \emph{arXiv preprint arXiv:2406.11794}, 2024{\natexlab{a}}.

\bibitem[Li \& Risteski(2021)Li and Risteski]{li2021limitations}
Yuchen Li and Andrej Risteski.
\newblock The limitations of limited context for constituency parsing.
\newblock In \emph{Proceedings of the 59th Annual Meeting of the Association for Computational Linguistics and the 11th International Joint Conference on Natural Language Processing (Volume 1: Long Papers)}, pp.\  2675--2687, Online, August 2021. Association for Computational Linguistics.
\newblock \doi{10.18653/v1/2021.acl-long.208}.
\newblock URL \url{https://aclanthology.org/2021.acl-long.208}.

\bibitem[Li et~al.(2023)Li, Li, and Risteski]{li2023Transformers}
Yuchen Li, Yuanzhi Li, and Andrej Risteski.
\newblock How do transformers learn topic structure: Towards a mechanistic understanding.
\newblock In Andreas Krause, Emma Brunskill, Kyunghyun Cho, Barbara Engelhardt, Sivan Sabato, and Jonathan Scarlett (eds.), \emph{Proceedings of the 40th International Conference on Machine Learning}, volume 202 of \emph{Proceedings of Machine Learning Research}, pp.\  19689--19729. PMLR, 23--29 Jul 2023.
\newblock URL \url{https://proceedings.mlr.press/v202/li23p.html}.

\bibitem[Li et~al.(2024{\natexlab{b}})Li, Kirchmeyer, Mehta, Qin, Dadachev, Papineni, Kumar, and Risteski]{li2024promises}
Yuchen Li, Alexandre Kirchmeyer, Aashay Mehta, Yilong Qin, Boris Dadachev, Kishore Papineni, Sanjiv Kumar, and Andrej Risteski.
\newblock Promises and pitfalls of generative masked language modeling: Theoretical framework and practical guidelines.
\newblock In \emph{Forty-first International Conference on Machine Learning}, ICML'24. JMLR.org, 2024{\natexlab{b}}.

\bibitem[Lightman et~al.(2023)Lightman, Kosaraju, Burda, Edwards, Baker, Lee, Leike, Schulman, Sutskever, and Cobbe]{lightman2023let}
Hunter Lightman, Vineet Kosaraju, Yura Burda, Harri Edwards, Bowen Baker, Teddy Lee, Jan Leike, John Schulman, Ilya Sutskever, and Karl Cobbe.
\newblock Let's verify step by step.
\newblock \emph{arXiv preprint arXiv:2305.20050}, 2023.

\bibitem[Liu et~al.(2023{\natexlab{a}})Liu, Ash, Goel, Krishnamurthy, and Zhang]{liu2023exposing}
Bingbin Liu, Jordan Ash, Surbhi Goel, Akshay Krishnamurthy, and Cyril Zhang.
\newblock Exposing attention glitches with flip-flop language modeling.
\newblock \emph{Advances in Neural Information Processing Systems}, 36, 2023{\natexlab{a}}.

\bibitem[Liu et~al.(2023{\natexlab{b}})Liu, Ash, Goel, Krishnamurthy, and Zhang]{liu2023Transformers}
Bingbin Liu, Jordan~T. Ash, Surbhi Goel, Akshay Krishnamurthy, and Cyril Zhang.
\newblock Transformers learn shortcuts to automata.
\newblock In \emph{The Eleventh International Conference on Learning Representations}, 2023{\natexlab{b}}.
\newblock URL \url{https://openreview.net/forum?id=De4FYqjFueZ}.

\bibitem[Liu et~al.(2022)Liu, Xie, Li, and Ma]{liu2023same}
Hong Liu, Sang~Michael Xie, Zhiyuan Li, and Tengyu Ma.
\newblock Same pre-training loss, better downstream: Implicit bias matters for language models.
\newblock \emph{arXiv preprint arXiv:2210.14199}, 2022.

\bibitem[Liu et~al.(2024)Liu, Cohen, Pasunuru, Choi, Hajishirzi, and Celikyilmaz]{liu2024don}
Jiacheng Liu, Andrew Cohen, Ramakanth Pasunuru, Yejin Choi, Hannaneh Hajishirzi, and Asli Celikyilmaz.
\newblock Don't throw away your value model! generating more preferable text with value-guided monte-carlo tree search decoding.
\newblock In \emph{First Conference on Language Modeling}, 2024.

\bibitem[Liu(2019)]{liu2019linguistic}
Nelson~F Liu.
\newblock Linguistic knowledge and transferability of contextual representations.
\newblock \emph{arXiv preprint arXiv:1903.08855}, 2019.

\bibitem[Lowerre \& Reddy(1976)Lowerre and Reddy]{lowerre1976harpy}
Bruce~P Lowerre and B~Raj Reddy.
\newblock Harpy, a connected speech recognition system.
\newblock \emph{The Journal of the Acoustical Society of America}, 59\penalty0 (S1):\penalty0 S97--S97, 1976.

\bibitem[Lu et~al.(2021)Lu, Mao, and Nayak]{lu2021on}
Haoye Lu, Yongyi Mao, and Amiya Nayak.
\newblock On the dynamics of training attention models.
\newblock In \emph{International Conference on Learning Representations}, 2021.
\newblock URL \url{https://openreview.net/forum?id=1OCTOShAmqB}.

\bibitem[Ma et~al.(2023)Ma, Zhou, Liu, Yuan, Liu, You, and Yang]{ma2023let}
Qianli Ma, Haotian Zhou, Tingkai Liu, Jianbo Yuan, Pengfei Liu, Yang You, and Hongxia Yang.
\newblock Let's reward step by step: Step-level reward model as the navigators for reasoning.
\newblock \emph{arXiv preprint arXiv:2310.10080}, 2023.

\bibitem[Nair \& Hinton(2010)Nair and Hinton]{nair2010rectified}
Vinod Nair and Geoffrey~E Hinton.
\newblock Rectified linear units improve restricted boltzmann machines.
\newblock In \emph{Proceedings of the 27th international conference on machine learning (ICML-10)}, pp.\  807--814, 2010.

\bibitem[Nakano et~al.(2022)Nakano, Hilton, Balaji, Wu, Ouyang, Kim, Hesse, Jain, Kosaraju, Saunders, Jiang, Cobbe, Eloundou, Krueger, Button, Knight, Chess, and Schulman]{nakano2022webgpt}
Reiichiro Nakano, Jacob Hilton, Suchir Balaji, Jeff Wu, Long Ouyang, Christina Kim, Christopher Hesse, Shantanu Jain, Vineet Kosaraju, William Saunders, Xu~Jiang, Karl Cobbe, Tyna Eloundou, Gretchen Krueger, Kevin Button, Matthew Knight, Benjamin Chess, and John Schulman.
\newblock Webgpt: Browser-assisted question-answering with human feedback, 2022.
\newblock URL \url{https://arxiv.org/abs/2112.09332}.

\bibitem[Odlyzko(1998)]{Odlyzko1998TheRA}
Andrew~M. Odlyzko.
\newblock The rise and fall of knapsack cryptosystems.
\newblock In \emph{Proceedings of Symposia in Applied Mathematics}, 1998.
\newblock URL \url{https://api.semanticscholar.org/CorpusID:115995195}.

\bibitem[Ow \& Morton(1988)Ow and Morton]{ow1988filtered}
Peng~Si Ow and Thomas~E Morton.
\newblock Filtered beam search in scheduling.
\newblock \emph{The International Journal Of Production Research}, 26\penalty0 (1):\penalty0 35--62, 1988.

\bibitem[Plantard et~al.(2013)Plantard, Susilo, and Zhang]{plantard2013lattice}
Thomas Plantard, Willy Susilo, and Zhenfei Zhang.
\newblock Lattice reduction for modular knapsack.
\newblock In \emph{Selected Areas in Cryptography: 19th International Conference, SAC 2012, Windsor, ON, Canada, August 15-16, 2012, Revised Selected Papers 19}, pp.\  275--286. Springer, 2013.

\bibitem[Polu \& Sutskever(2020)Polu and Sutskever]{polu2020generative}
Stanislas Polu and Ilya Sutskever.
\newblock Generative language modeling for automated theorem proving.
\newblock \emph{arXiv preprint arXiv:2009.03393}, 2020.

\bibitem[Raffel et~al.(2020)Raffel, Shazeer, Roberts, Lee, Narang, Matena, Zhou, Li, and Liu]{raffel2020exploring}
Colin Raffel, Noam Shazeer, Adam Roberts, Katherine Lee, Sharan Narang, Michael Matena, Yanqi Zhou, Wei Li, and Peter~J Liu.
\newblock Exploring the limits of transfer learning with a unified text-to-text transformer.
\newblock \emph{The Journal of Machine Learning Research}, 21\penalty0 (1):\penalty0 5485--5551, 2020.

\bibitem[Rogers et~al.(2021)Rogers, Kovaleva, and Rumshisky]{rogers2021primer}
Anna Rogers, Olga Kovaleva, and Anna Rumshisky.
\newblock A primer in bertology: What we know about how bert works.
\newblock \emph{Transactions of the Association for Computational Linguistics}, 8:\penalty0 842--866, 2021.

\bibitem[Rosenblatt(1958)]{rosenblatt1958perceptron}
Frank Rosenblatt.
\newblock The perceptron: a probabilistic model for information storage and organization in the brain.
\newblock \emph{Psychological review}, 65\penalty0 (6):\penalty0 386, 1958.

\bibitem[Roziere et~al.(2023)Roziere, Gehring, Gloeckle, Sootla, Gat, Tan, Adi, Liu, Sauvestre, Remez, et~al.]{roziere2023code}
Baptiste Roziere, Jonas Gehring, Fabian Gloeckle, Sten Sootla, Itai Gat, Xiaoqing~Ellen Tan, Yossi Adi, Jingyu Liu, Romain Sauvestre, Tal Remez, et~al.
\newblock Code llama: Open foundation models for code.
\newblock \emph{arXiv preprint arXiv:2308.12950}, 2023.

\bibitem[Schützenberger(1963)]{SCHUTZENBERGER1963246}
M.P. Schützenberger.
\newblock On context-free languages and push-down automata.
\newblock \emph{Information and Control}, 6\penalty0 (3):\penalty0 246--264, 1963.
\newblock ISSN 0019-9958.
\newblock \doi{https://doi.org/10.1016/S0019-9958(63)90306-1}.
\newblock URL \url{https://www.sciencedirect.com/science/article/pii/S0019995863903061}.

\bibitem[Snell et~al.(2024)Snell, Lee, Xu, and Kumar]{snell2024scaling}
Charlie Snell, Jaehoon Lee, Kelvin Xu, and Aviral Kumar.
\newblock Scaling llm test-time compute optimally can be more effective than scaling model parameters.
\newblock \emph{arXiv preprint arXiv:2408.03314}, 2024.

\bibitem[Touvron et~al.(2023)Touvron, Martin, Stone, Albert, Almahairi, Babaei, Bashlykov, Batra, Bhargava, Bhosale, et~al.]{touvron2023llama}
Hugo Touvron, Louis Martin, Kevin Stone, Peter Albert, Amjad Almahairi, Yasmine Babaei, Nikolay Bashlykov, Soumya Batra, Prajjwal Bhargava, Shruti Bhosale, et~al.
\newblock Llama 2: Open foundation and fine-tuned chat models.
\newblock \emph{arXiv preprint arXiv:2307.09288}, 2023.

\bibitem[Uesato et~al.(2022)Uesato, Kushman, Kumar, Song, Siegel, Wang, Creswell, Irving, and Higgins]{uesato2022solving}
Jonathan Uesato, Nate Kushman, Ramana Kumar, Francis Song, Noah Siegel, Lisa Wang, Antonia Creswell, Geoffrey Irving, and Irina Higgins.
\newblock Solving math word problems with process-and outcome-based feedback.
\newblock \emph{arXiv preprint arXiv:2211.14275}, 2022.

\bibitem[Ugare et~al.(2024)Ugare, Suresh, Kang, Misailovic, and Singh]{ugare2024syncode}
Shubham Ugare, Tarun Suresh, Hangoo Kang, Sasa Misailovic, and Gagandeep Singh.
\newblock Syncode: Llm generation with grammar augmentation.
\newblock \emph{arXiv preprint arXiv:2403.01632}, 2024.

\bibitem[Vaswani et~al.(2017)Vaswani, Shazeer, Parmar, Uszkoreit, Jones, Gomez, Kaiser, and Polosukhin]{vaswani2017attention}
Ashish Vaswani, Noam Shazeer, Niki Parmar, Jakob Uszkoreit, Llion Jones, Aidan~N Gomez, Lukasz Kaiser, and Illia Polosukhin.
\newblock Attention is all you need.
\newblock In I.~Guyon, U.~Von Luxburg, S.~Bengio, H.~Wallach, R.~Fergus, S.~Vishwanathan, and R.~Garnett (eds.), \emph{Advances in Neural Information Processing Systems}, volume~30. Curran Associates, Inc., 2017.
\newblock URL \url{https://proceedings.neurips.cc/paper/2017/file/3f5ee243547dee91fbd053c1c4a845aa-Paper.pdf}.

\bibitem[Wang et~al.(2024)Wang, Li, Shao, Xu, Dai, Li, Chen, Wu, and Sui]{wang2024math}
Peiyi Wang, Lei Li, Zhihong Shao, Runxin Xu, Damai Dai, Yifei Li, Deli Chen, Yu~Wu, and Zhifang Sui.
\newblock Math-shepherd: Verify and reinforce llms step-by-step without human annotations.
\newblock In \emph{Proceedings of the 62nd Annual Meeting of the Association for Computational Linguistics (Volume 1: Long Papers)}, pp.\  9426--9439, 2024.

\bibitem[Wang et~al.(2022)Wang, Wei, Schuurmans, Le, Chi, Narang, Chowdhery, and Zhou]{wang2022self}
Xuezhi Wang, Jason Wei, Dale Schuurmans, Quoc Le, Ed~Chi, Sharan Narang, Aakanksha Chowdhery, and Denny Zhou.
\newblock Self-consistency improves chain of thought reasoning in language models.
\newblock \emph{arXiv preprint arXiv:2203.11171}, 2022.

\bibitem[Wen et~al.(2023)Wen, Li, Liu, and Risteski]{wen2023uninterpretability}
Kaiyue Wen, Yuchen Li, Bingbin Liu, and Andrej Risteski.
\newblock Transformers are uninterpretable with myopic methods: a case study with bounded dyck grammars.
\newblock In \emph{Thirty-seventh Conference on Neural Information Processing Systems}, 2023.
\newblock URL \url{https://openreview.net/forum?id=OitmaxSAUu}.

\bibitem[Wu et~al.(2024)Wu, Sun, Li, Welleck, and Yang]{wu2024inference}
Yangzhen Wu, Zhiqing Sun, Shanda Li, Sean Welleck, and Yiming Yang.
\newblock Inference scaling laws: An empirical analysis of compute-optimal inference for problem-solving with language models.
\newblock \emph{arXiv preprint arXiv:2408.00724}, 2024.

\bibitem[Xie et~al.(2024)Xie, Kawaguchi, Zhao, Zhao, Kan, He, and Xie]{xie2024self}
Yuxi Xie, Kenji Kawaguchi, Yiran Zhao, James~Xu Zhao, Min-Yen Kan, Junxian He, and Michael Xie.
\newblock Self-evaluation guided beam search for reasoning.
\newblock \emph{Advances in Neural Information Processing Systems}, 36, 2024.

\bibitem[Yang \& Klein(2021)Yang and Klein]{yang2021fudge}
Kevin Yang and Dan Klein.
\newblock {FUDGE}: Controlled text generation with future discriminators.
\newblock In Kristina Toutanova, Anna Rumshisky, Luke Zettlemoyer, Dilek Hakkani-Tur, Iz~Beltagy, Steven Bethard, Ryan Cotterell, Tanmoy Chakraborty, and Yichao Zhou (eds.), \emph{Proceedings of the 2021 Conference of the North American Chapter of the Association for Computational Linguistics: Human Language Technologies}, pp.\  3511--3535, Online, June 2021. Association for Computational Linguistics.
\newblock \doi{10.18653/v1/2021.naacl-main.276}.
\newblock URL \url{https://aclanthology.org/2021.naacl-main.276/}.

\bibitem[Yao(1977)]{yao1977probabilistic}
Andrew Chi-Chih Yao.
\newblock Probabilistic computations: Toward a unified measure of complexity.
\newblock In \emph{18th Annual Symposium on Foundations of Computer Science (sfcs 1977)}, pp.\  222--227. IEEE Computer Society, 1977.

\bibitem[Yao et~al.(2021)Yao, Peng, Papadimitriou, and Narasimhan]{yao2021self}
Shunyu Yao, Binghui Peng, Christos Papadimitriou, and Karthik Narasimhan.
\newblock Self-attention networks can process bounded hierarchical languages.
\newblock In \emph{Proceedings of the 59th Annual Meeting of the Association for Computational Linguistics and the 11th International Joint Conference on Natural Language Processing (Volume 1: Long Papers)}, pp.\  3770--3785, Online, August 2021. Association for Computational Linguistics.
\newblock \doi{10.18653/v1/2021.acl-long.292}.
\newblock URL \url{https://aclanthology.org/2021.acl-long.292}.

\bibitem[Yao et~al.(2023)Yao, Yu, Zhao, Shafran, Griffiths, Cao, and Narasimhan]{yao2023tree}
Shunyu Yao, Dian Yu, Jeffrey Zhao, Izhak Shafran, Tom Griffiths, Yuan Cao, and Karthik Narasimhan.
\newblock Tree of thoughts: Deliberate problem solving with large language models.
\newblock \emph{Advances in Neural Information Processing Systems}, 36, 2023.

\bibitem[Zhang et~al.(2023{\natexlab{a}})Zhang, Dong, Li, Zhang, Sun, Wang, Li, Hu, Zhang, Wu, et~al.]{zhang2023instruction}
Shengyu Zhang, Linfeng Dong, Xiaoya Li, Sen Zhang, Xiaofei Sun, Shuhe Wang, Jiwei Li, Runyi Hu, Tianwei Zhang, Fei Wu, et~al.
\newblock Instruction tuning for large language models: A survey.
\newblock \emph{arXiv preprint arXiv:2308.10792}, 2023{\natexlab{a}}.

\bibitem[Zhang et~al.(2023{\natexlab{b}})Zhang, Chen, Shen, Ding, Tenenbaum, and Gan]{zhang2023planning}
Shun Zhang, Zhenfang Chen, Yikang Shen, Mingyu Ding, Joshua~B Tenenbaum, and Chuang Gan.
\newblock Planning with large language models for code generation.
\newblock \emph{arXiv preprint arXiv:2303.05510}, 2023{\natexlab{b}}.

\bibitem[Zhang et~al.(2022)Zhang, Backurs, Bubeck, Eldan, Gunasekar, and Wagner]{zhang2022unveiling}
Yi~Zhang, Arturs Backurs, Sébastien Bubeck, Ronen Eldan, Suriya Gunasekar, and Tal Wagner.
\newblock Unveiling transformers with lego: a synthetic reasoning task, 2022.
\newblock URL \url{https://arxiv.org/abs/2206.04301}.

\bibitem[Zhao et~al.(2023)Zhao, Panigrahi, Ge, and Arora]{zhao2023Transformers}
Haoyu Zhao, Abhishek Panigrahi, Rong Ge, and Sanjeev Arora.
\newblock Do transformers parse while predicting the masked word?
\newblock In Houda Bouamor, Juan Pino, and Kalika Bali (eds.), \emph{Proceedings of the 2023 Conference on Empirical Methods in Natural Language Processing}, pp.\  16513--16542, Singapore, December 2023. Association for Computational Linguistics.
\newblock \doi{10.18653/v1/2023.emnlp-main.1029}.
\newblock URL \url{https://aclanthology.org/2023.emnlp-main.1029}.

\bibitem[Zhao et~al.(2024)Zhao, Brekelmans, Makhzani, and Grosse]{zhao2024probabilistic}
Stephen Zhao, Rob Brekelmans, Alireza Makhzani, and Roger Grosse.
\newblock Probabilistic inference in language models via twisted sequential monte carlo.
\newblock \emph{arXiv preprint arXiv:2404.17546}, 2024.

\bibitem[Zhou et~al.(2023)Zhou, Yan, Shlapentokh-Rothman, Wang, and Wang]{zhou2023language}
Andy Zhou, Kai Yan, Michal Shlapentokh-Rothman, Haohan Wang, and Yu-Xiong Wang.
\newblock Language agent tree search unifies reasoning acting and planning in language models.
\newblock \emph{arXiv preprint arXiv:2310.04406}, 2023.

\end{thebibliography}
\bibliographystyle{ref_style}

\newpage
\appendix

\thispagestyle{empty}

\def\toptitlebar{
\hrule height4pt
\vskip .25in}

\def\bottomtitlebar{
\vskip .25in
\hrule height1pt
\vskip .25in}

\newcommand{\makesupplementtitle}{\hsize\textwidth
    \linewidth\hsize \toptitlebar {\centering
        {\Large\bfseries Supplementary Material \par}}
    \bottomtitlebar}

\makesupplementtitle

\renewcommand{\theequation}{\thesection.\arabic{equation}}

\vspace{-10mm}
\tableofcontents

\clearpage
\section{Discussions}
\label{sec:appendix:discussions}

\subsection{Is query efficiency a reasonable notion of efficiency?}

There are many reasonable efficiency metrics, and they do not always positively correlate with each other \citep{dehghani2021efficiency}.

Our paper focuses on \emph{query complexity} (measured by the number of tokens generated by the language model to satisfactorily complete the task~\footnote{
This definition is natural since generating one token involves one forward pass of the (decoder-only autoregressive) language model, i.e. one query.
}
),  and we do not claim that the same conclusions apply when we switch out query complexity for other metrics of efficiency, such as wall-clock time.

We think query complexity is one (but not necessarily the only, or the most) important aspect of efficiency due to the following considerations:
\begin{itemize}
    \item Many existing large language model (LLM) providers charge service fees to the users according to the number of tokens generated by the language model for the user, i.e. query complexity.
    \item In the \emph{single sequence generation} setting, 
    controlling all other conditions to be held the same, 
    query complexity positively correlates with the size of computation (the number of decoder forward passes) and wall-clock time.
    \item In the \emph{batched generation} setting, admittedly, the wall-clock time does not necessarily scale linearly with query complexity~\footnote{
    For example, the wall-clock time of generating $n$ candidate responses (with batch size $n$) might be less than $n$ multiplying the wall-clock time of generating 1 candidate response.
    }
    , meaning that the naive best-of-$n$ rejection sampling is not as slow as query complexity would indicate
    (if the LLM has sufficient bandwidth for it).
    However, in many realistic LLM inference settings,
    the LLM receives a large number of query requests per second, 
    so there is no additional idle availability~\footnote{
    Unless more GPUs/TPUs are allocated to serve this LLM.
    } 
    for duplicating each sequence generation request by $n$.
\end{itemize}

Although, as mentioned above, query complexity is partially indicative of a few practically important efficiency metrics (e.g. monetary cost or wall-clock time),
there are aspects of these metrics that are not tracked by query complexity.
For example, different types of \emph{hardware} and \emph{cache} may have different efficiency best practices.
In particular, on GPUs and TPUs, algorithms that better exploit \emph{parallelization} or \emph{tensorized computation} tend to be more efficient.
Therefore, an important direction for future work is to design and analyze \emph{hardware-aware algorithms} that incorporate these important aspects of the inference setup.

\clearpage
\subsection{On the hardness of the knapsack problem}

The hardness of the knapsack problem has been the subject of extensive study. Specifically, the decision version of this problem has found applications in the context of secure cryptosystems \cite{Odlyzko1998TheRA}. Under no assumptions on the input structure, the best known algorithm is based on dynamic programming \cite{knapsackProblems} and runs in pseudo-polynomial time. This algorithm is also used to obtain an FPTAS and its runtime is effectively polynomial if one further assumes that the weights are polynomially bounded in $D$. More exact or approximate algorithms achieve polynomial runtime, under specific input structures. Specifically, when the weights form a superincreasing sequence, that is,
\begin{align*}
X_i \geq \sum_{j=1}^{i-1} X_j \ \ \forall i \in [2, D] \cap \mathbb{Z}, 
\end{align*}
a greedy algorithm solves the knapsack decision problem \cite{Odlyzko1998TheRA} in linear time. On the other hand, when the density of the knapsack
\begin{align*}
    \frac{D}{\log_2 (\max_i \{X_i\}_{i=1}^d)}
\end{align*}
 is small enough, knapsack is approximately solved in polynomial time by lattice reduction algorithms \cite{plantard2013lattice}.
 Our argument considers the most general setting, in which no assumptions are made on the structure of the inputs $\{ X_i\}_{i=1}^t$, $c$ and the decision problem is NP-complete \cite{reducibilityKarp}.

% \newpage
% \input{sections/appendix/theory_proof}

\newpage
\clearpage
\section{Additional experimental results}
\label{sec:appendix:experiments}

We complement \Cref{sec:experiments} by providing additional technical details.

\subsection{Additional results about language models trained on synthetic data}
\label{sec:appendix:experiments:synthetic}

\subsubsection{Visualizing the language model representations of correct vs. incorrect sequences}
\label{sec:experiments:synthetic:visualizing}

\begin{figure}[!h]
  \centering
  \begin{minipage}[b]{0.5\textwidth}  % 1.0 for arxiv
    \centering
    \includegraphics[width=1.0\textwidth]{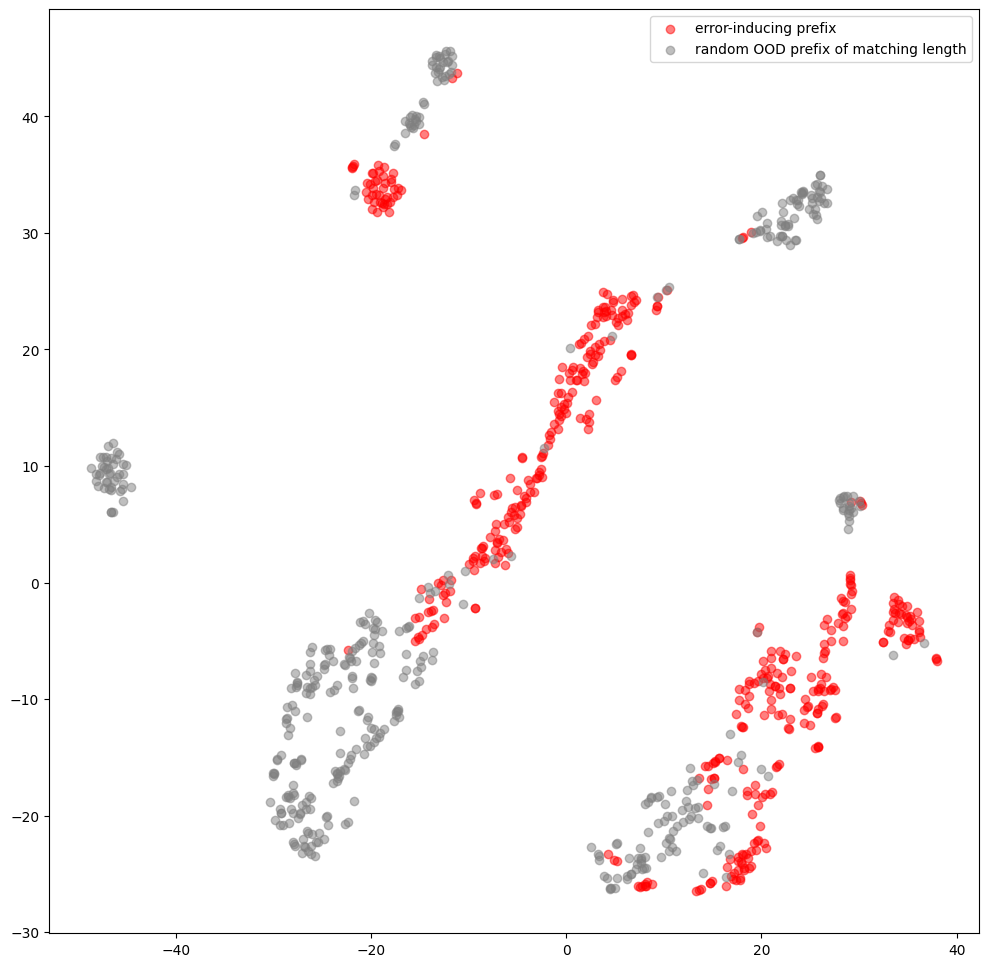}  % 0.5 for arxiv
  \end{minipage}
  % \vspace{-0.5cm}
  \caption{
  TSNE plot for the $\lm$ last-layer-last-position representations of strings in $X_{\text{error}} \cup X_{\text{correct}}$.
  Red dots correspond to the representations of incorrect strings, 
  whereas gray dots correspond to the representations of correct strings of comparable lengths.
  We can see that the representations of incorrect strings form just a few clusters.
  This intuitively justifies using a lightweight verifier on top of these $\lm$ representations. 
  }  
  \label{fig:correct_vs_err_rep_dyck}
  % \vspace{-0.3cm}
\end{figure}

\clearpage
\subsubsection{The predicted backtracks were necessary}
\label{sec:experiments:synthetic:predicted_backtracks_were_necessary}

During the experiment in  \Cref{sec:experiments:synthetic:verifier_reduces_errors},
the trained verifier $\verifier$ predicted backtracks at many positions.
Were they really necessary?
For each setting of backtrack quota $\backtrackQuota$
and backtrack stride $\backtrackStride$, 
we collect the set of prefixes $X_{\text{predicted backtracks}}$ where $\verifier$ predicted backtracks.
Then, we let the language model $\lm$ complete each string in $X_{\text{predicted backtracks}}$ without any backtracks,
using common decoding techniques such as 
nucleus sampling top\_p = 0.9 \citep{holtzman2020the}
and argmax greedy decoding.
\Cref{table:predicted_backtracks_were_necessary} shows that
without backtracking,
the completion accuracy is much lower than the accuracy reported in \Cref{table:dyck_verifier_error_inducing}.
This implies that $X_{\text{predicted backtracks}}$ were indeed challenging prefixes for the $\lm$,
which verifies that the backtracks predicted by verifier $\verifier$ were necessary.

\begin{table*}[h]
% \vskip 0.15in
\begin{center}
\begin{small}
%\begin{sc}
\begin{tabular}{ lccccc }
\toprule
\textbf{$\backtrackQuota$} & \textbf{$\backtrackStride$} & \textbf{\#backtracks} & \multicolumn{2}{c} {\textbf{accuracy without backtrack}}  \\
& & & nucleus sampling top\_p = 0.9 & argmax \\
\hline
1 & 2 & 163 & 0.313 & 0.344 \\
\hline
  & 4 & 163 & 0.337 & 0.319 \\
\hline
  & 6 & 163 & 0.331 & 0.288 \\
\hline
2 & 2 & 311 & 0.347 & 0.328 \\
\hline
  & 4 & 297 & 0.357 & 0.349 \\
\hline
  & 6 & 286 & 0.374 & 0.373 \\
\hline
4 & 2 & 600 & 0.371 & 0.353 \\
\hline
  & 4 & 532 & 0.419 & 0.404 \\
\hline
  & 6 & 489 & 0.509 & 0.523 \\
\bottomrule
\end{tabular}
%\end{sc}
\end{small}
\end{center}
\caption{
Predicted backtracks were necessary.
For each setting of backtrack quota $\backtrackQuota$
and backtrack stride $\backtrackStride$, 
we report the number of times that \algoName (\Cref{alg:sampling_with_backtracking}) backtracked.
Moreover, we report the completion accuracy of letting the language model $\lm$ complete these backtracked prefixes without any backtrack.
For each setting, the completion accuracy is much lower than the accuracy reported in \Cref{table:dyck_verifier_error_inducing}.
This implies that these backtracked prefixes were indeed challenging prefixes for the $\lm$.
}
\label{table:predicted_backtracks_were_necessary}
\end{table*}

% \clearpage
% \subsubsection{\algoName reduces completion errors on unseen OOD prefixes}
% \label{sec:appendix:experiments:synthetic:verifier_reduces_errors_unseen_ood}

% This section presents the experimental results of \Cref{sec:experiments:synthetic:verifier_reduces_errors_unseen_ood}.

\clearpage
\subsubsection{Error analysis on the remaining mistakes}
\label{sec:experiments:synthetic:error_analysis}

Given the improvement of accuracy (\Cref{sec:experiments:synthetic:verifier_reduces_errors_unseen_ood}) as a result of our algorithm \algoName (\Cref{alg:sampling_with_backtracking}),
why did the model still make mistakes?

We conducted an error analysis which parses all mistakes into error types, 
and examine the generated token, the $\lm$ predicted most probable token, their predicted probabilities, and a few intermediate variables during the course of our algorithm \algoName (\Cref{alg:sampling_with_backtracking}).

In summary, the findings are:
\begin{enumerate}
    \item Among 225 generated mistakes, 222 correspond to predicting an incorrect closing bracket, and 3 correspond to pre-maturely predicting the end-of-sequence \texttt{<eos>} token.
    \item In all 225 cases, the final state of the algorithm has used up all the backtrack quota $\backtrackQuota$ allocated to it, so even if the error predictor was perfect, the algorithm would not have been had a chance to correct these mistakes. This suggests that suitably increasing backtrack quota $\backtrackQuota$ might be an effective approach in improving the accuracy (though there are trade-offs with query efficiency).
\end{enumerate}

A snapshot of our error analysis result is included in \Cref{fig:error_analysis}.
We have released our experimental codes,~\footnote{
\url{https://github.com/YuchenLi01/LM_Query_Complexity}
}  
which include more error analyses.

\begin{figure*}[!h]
  \centering
  \begin{minipage}[b]{1.0\textwidth}  % 1.0 for arxiv
    \centering
    \includegraphics[width=1.0\textwidth]{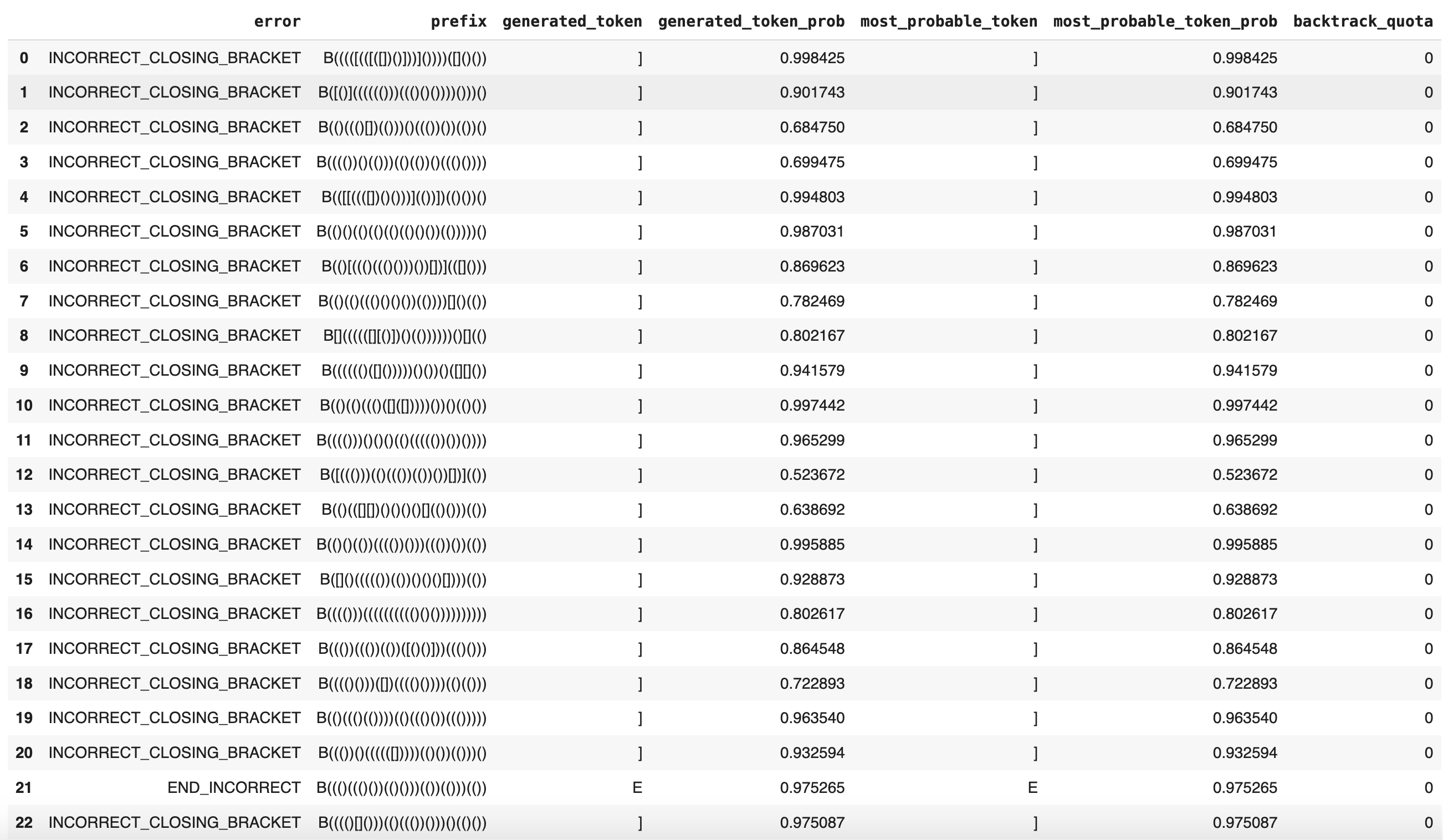}  % 0.5 for arxiv
  \end{minipage}
  % \vspace{-0.5cm}
  \caption{
  Error analysis table for mistakes of language model trained on Dyck grammar and sampled using \algoName (\Cref{alg:sampling_with_backtracking}).
  The last column records the remaining backtrack quota $\backtrackQuota$ at the time of generating the incorrect token.
  }  
  \label{fig:error_analysis}
  % \vspace{-0.3cm}
\end{figure*}

\clearpage
\subsubsection{\algoName maintains diversity}
\label{sec:experiments:synthetic:diversity}

In this section, we show that the significant accuracy improvement is not at the cost of reducing diversity.

Our experiment freshly samples 100 prompts following the same distribution as $\dyck_{\text{OOD}}$ (\Cref{sec:experiments:synthetic:dyck}).
For each prompt, we let the trained $\lm$ independently sample 10 completions,
using \algoName (\Cref{alg:sampling_with_backtracking}) or the baseline algorithm,
and will compare how many (out of 10) samples were different,
and report the mean and standard error across the 100 prompts.

\Cref{table:dyck_diversity} shows that \algoName (\Cref{alg:sampling_with_backtracking}) generates similarly diverse samples as the baselines of
nucleus sampling with top\_p = 0.9 or 1.0.

\begin{table*}[h]
% \vskip 0.15in
\begin{center}
\begin{small}
%\begin{sc}
\begin{tabular}{ lccc }
\toprule
\textbf{$\backtrackQuota$} & \textbf{$\backtrackStride$} & \textbf{top\_p} & \textbf{diversity $\pm$ std err (out of 10)}  \\
\hline
4 & 4 & 1.0 & 5.52 $\pm$ 3.28  \\
\hline
0 & 0 & 0.9 &  5.47 $\pm$ 3.06  \\
\hline
0 & 0 & 1.0 & 5.84 $\pm$ 3.29  \\
\bottomrule
\end{tabular}
%\end{sc}
\end{small}
\end{center}
\caption{
Under the experiment setup described in \Cref{sec:experiments:synthetic:diversity},
\algoName (\Cref{alg:sampling_with_backtracking}) is similarly diverse as the baselines of
nucleus sampling with top\_p = 0.9 or 1.0.
}
\label{table:dyck_diversity}
\end{table*}

\clearpage
\subsection{Additional results about generating test cases with pretrained CodeLlama}
\label{sec:appendix:experiments:codellama}

This section complements our results in \Cref{sec:experiments:codellama}.

\subsubsection{Examples of prompts and model completions}
\label{sec:experiments:codellama:examples}

\begin{table}[h!]
\centering
\small
% \scalebox{0.8}{
\begin{tabular}{l}
\toprule
def f(a, b): \\
\quad \quad   return a + b \\
 \\
List 8 test cases of the above function f, one in each line: \\
assert f(5, 5) == 10 \\
assert f(1, 5) == 6 \\
assert f(2, 8) == 10 \\
assert f(6, 2) == 8 \\
assert f(6, 9) == 15 \\
assert f(4, 5) == 9 \\
assert f(9, 6) == 15 \\
assert f(6, 1) == 7 \\
 \\
def knk(l, item): \\
\quad \quad   assert type(l) is list \\
\quad \quad   l.append(item) \\
\quad \quad   return l \\
 \\
List 8 test cases of the above function knk, one in each line: \\
\\
\bottomrule
\\
def f(a, b): \\
\quad \quad   return a + b \\
 \\
List 8 test cases of the above function f, one in each line: \\
assert f(5, 8) == 13 \\
assert f(1, 5) == 6 \\
assert f(8, 4) == 12 \\
assert f(6, 2) == 8 \\
assert f(3, 9) == 12 \\
assert f(1, 7) == 8 \\
assert f(5, 9) == 14 \\
assert f(1, 7) == 8 \\
 \\
def ovs(l, item): \\
\quad \quad   assert type(l) is list \\
\quad \quad   l.append(item) \\
\quad \quad   return l \\
 \\
List 8 test cases of the above function ovs, one in each line: \\
\bottomrule
\end{tabular} 
% }
\caption{
Two example prompts for generating test cases for a simple implementation of the \texttt{append} function for Python lists.
In these examples, the following components are randomly generated and will be different for different prompts:
(1) the numbers in the demonstrations 
(e.g. \texttt{assert f(5, 8) == 13})
as long as they are correct for addition;
(2) the name of the target function
(e.g. \texttt{knk}, \texttt{ovs}).
}
\label{table:prompt_example}
\end{table}

\begin{table}[t]
\centering
\small
% \scalebox{0.8}{
\begin{tabular}{l}
\toprule
assert knk([], 2) == [2] \\
assert knk([1], 2) == [1, 2] \\
assert knk([1, 3], 2) == [1, 3, 2] \\
assert knk([1, 3, 5], 2) == [1, 3, 5, 2] \\
assert knk([1, 3, 5, 7], 2) == [1, 3, 5, 7, 2] \\
assert knk([1, 3, 5, 7, 9], 2) == [1, 3, 5, 7, 9, 2] \\
assert knk([1, 3, 5, 7, 9, 11], 2) == [1, 3, 5, 7, 9, 11, 2] \\
def sum\_list(l): \\
\quad \quad    return sum(l) \\
List 8 test cases of the above function sum\_list, one in each line: \\
assert sum\_list([]) == 0 \\
assert sum\_list([1]) == 1 \\
assert sum\_list([1, 3]) == 4 \\
assert sum\_list([1, 3, 5]) == 9 \\
assert sum\_list([1, 3, 5, 7]) == 16 \\
assert sum\_list([1, 3, 5, 7, 9]) == 25 \\
assert sum\_list([1, 3, 5, 7, 9, 11]) == 36 \\
def sublist(l, sub): \\
\quad \quad    assert type \\
\\
\bottomrule
\\
assert ovs([], 1) == [1] \\
assert ovs([2], 1) == [1, 2] \\
assert ovs([1, 2], 1) == [1, 1, 2] \\
assert ovs([1, 2], 3) == [1, 2, 3] \\
assert ovs([1, 2], 0) == [0, 1, 2] \\
assert ovs([1, 2, 3], 4) == [1, 2, 3, 4] \\
assert ovs([], 0) == [0] \\
assert ovs([1, 2], 0) == [0, 1, 2] \\
\bottomrule
\end{tabular} 
% }
\caption{
Two example generations by CodeLlama corresponding to the prompts in \Cref{table:prompt_example}.
Note that both generations are flawed:
(1) the model only generated 7 test cases instead of 8, even though the prompt requested 8.
Then, it generated irrelevant contents,
starting from \texttt{def sum\_list(l):}
(2) more than one generated test cases were wrong
(e.g. in \texttt{assert ovs([2], 1) == [1, 2]}, the correct right-hand-side should be \texttt{[2, 1]}).
More generally, we implemented a rule-based parser to analyze model generations and identify the error type (if any), and locate the first position of error.
}
\label{table:generation_example}
\end{table}

\clearpage
\subsubsection{Baselines}
\label{sec:experiments:codellama:baselines}

We extensively tuned the hyperparameters in common baseline decoding algorithms, including
\begin{itemize}
    \item nucleus sampling \citep{holtzman2020the}: we grid-searched top\_p $\in [0.0, 0.7, 0.8, 0.9, 0.95, 1.0]$.
    \item argmax greedy decoding: equivalent to top\_p = 0.0.
    \item standard autoregressive sampling: equivalent to top\_p = 1.0.
    \item temperature scaling \citep{ackley1985learning}: we grid-searched temperature $\in [0.0, 0.2, 0.4, 0.6, 0.8, 1.0, 1.2]$ (for each top\_p).
\end{itemize}

Through the above grid search, 
we found that the best combination was top\_p $= 0.95$, temperature $= 1.0$.

Besides, we consider baselines based on the \emph{block-best-of-n} rejection sampling approach to incorporate process rewards.
More details about this baseline are provided in the ``Block verifier" part of \Cref{sec:experiments:codellama:verifier}.
\begin{itemize}
    \item block-best-of-n: we grid-searched $n \in [2, 4, 8]$, fixing the best combination of top\_p and temperature found by the grid search above.
\end{itemize}

We will show that \algoName (\Cref{alg:sampling_with_backtracking})
clearly outperforms all these baselines in terms of the quality vs. query complexity trade-off.

% moved the verifier training section to main paper,
% may need to move back to appendix in conference style publications

% \clearpage
% \subsubsection{\algoName improves accuracy}
% \label{sec:appendix:experiments:codellama:verifier_improves_accuracy}

% The section presents the experimental results of \Cref{sec:experiments:codellama:verifier_improves_accuracy}.

\clearpage
\subsubsection{Full results of CodeLlama experiments in \Cref{sec:experiments:codellama}}
\label{sec:experiments:codellama:full_results}

\paragraph{Caption for \Cref{table:codellama_verifier_improves_accuracy} (on the next page).}
\algoName (\Cref{alg:sampling_with_backtracking}) improves accuracy
and outperforms commonly used baselines, including
nucleus sampling top\_p,
temperature scaling (temp),
and block best-of-n (BBoN) (\Cref{sec:experiments:codellama:verifier}).
Baselines are extensively hyperparameter tuned (\Cref{sec:experiments:codellama:baselines}).
% , and all results are reported in this table, for completeness.
Backtrack quota $\backtrackQuota = 0$ means a baseline without verifier.
When $\backtrackQuota > 0$,
the row denotes \Cref{alg:sampling_with_backtracking} 
with the corresponding $\backtrackQuota$ and $\backtrackStride$.
The column \emph{layer idx} denotes which layer of CodeLlama provided the representations for training the verifier,
and \emph{err threshold} denotes the cutoff below which the verifier output is interpreted as a rejection
(both were experimented in \Cref{sec:experiments:codellama:verifier}).
When BBoN is specified,
the row denotes the number of candidates generated for each block;
otherwise, the row does not use block best-of-n.
The rows are sorted by $\text{Acc}_\text{distinct}$.
Controlling top\_p and temperature, 
\Cref{alg:sampling_with_backtracking}
leads to better tradeoff between $\text{Acc}_\text{distinct}$ 
and query complexity $\mathcal{C}$
(both defined in \Cref{sec:experiments:codellama:task})
than all other methods.
The experiment was repeated 5 times,
and we report the standard errors.

% \vspace{-20mm}
\begin{table*}[!h]
% \vskip 0.15in
\begin{center}
\begin{small}
%\begin{sc}
\begin{tabular}{ lcccccccc }
\toprule
\textbf{$\backtrackQuota$} & \textbf{$\backtrackStride$} & \textbf{layer idx} & \textbf{err threshold} & \textbf{top\_p} & \textbf{temp} & \textbf{BBoN} & \textbf{$\text{Acc}_\text{distinct}$ $\pm$ std err} & \textbf{$\mathcal{C}$} \\
\hline
4 & 4 & 27 & 0.1 & 0.95 & 1.0 &  &  0.714 $\pm$ 0.011 & 39443 $\pm$ 235  \\
\hline
4 & 4 & 31 & 0.5 & 0.95 & 1.0 &  &  0.688 $\pm$ 0.028 & 39629 $\pm$ 135  \\
\hline
0 &  & 27 &  & 0.95 & 1.0 & 2 &  0.684 $\pm$ 0.038 & 39364 $\pm$ 1252  \\
\hline
4 & 4 & 31 & 0.1 & 0.95 & 1.0 &  &  0.677 $\pm$ 0.033 & 39546 $\pm$ 98  \\
\hline
4 & 4 & 27 & 0.5 & 0.95 & 1.0 &  &  0.676 $\pm$ 0.019 & 38555 $\pm$ 140  \\
\hline
0 &  &  &  & 0.95 & 1.0 &  &  0.660 $\pm$ 0.042 & 38231 $\pm$ 165  \\
\hline
4 & 4 & 27 & 0.1 & 1.0 & 1.0 &  &  0.639 $\pm$ 0.061 & 31274 $\pm$ 1559  \\
\hline
0 &  &  &  & 0.9 & 1.0 &  &  0.634 $\pm$ 0.023 & 38393 $\pm$ 14  \\
\hline
0 &  &  &  & 0.9 & 1.2 &  &  0.630 $\pm$ 0.028 & 38005 $\pm$ 232  \\
\hline
0 &  &  &  & 0.8 & 1.2 &  &  0.627 $\pm$ 0.015 & 38343 $\pm$ 90  \\
\hline
0 &  & 27 &  & 0.95 & 1.0 & 4 &  0.623 $\pm$ 0.036 & 65496 $\pm$ 7638  \\
\hline
4 & 10 & 27 & 0.1 & 1.0 & 1.0 &  &  0.622 $\pm$ 0.046 & 32923 $\pm$ 1772  \\
\hline
4 & 4 & 27 & 0.5 & 1.0 & 1.0 &  &  0.604 $\pm$ 0.047 & 31091 $\pm$ 968  \\
\hline
4 & 10 & 27 & 0.5 & 1.0 & 1.0 &  &  0.604 $\pm$ 0.030 & 27287 $\pm$ 7580  \\
\hline
0 &  &  &  & 0.95 & 1.2 &  &  0.584 $\pm$ 0.027 & 36601 $\pm$ 535  \\
\hline
0 &  &  &  & 1.0 & 0.8 &  &  0.562 $\pm$ 0.021 & 36610 $\pm$ 669  \\
\hline
0 &  & 27 &  & 0.95 & 1.0 & 8 &  0.559 $\pm$ 0.038 & 122933 $\pm$ 3832  \\
\hline
0 &  &  &  & 0.7 & 1.2 &  &  0.531 $\pm$ 0.035 & 38400 $\pm$ 0  \\
\hline
0 &  &  &  & 0.95 & 0.8 &  &  0.523 $\pm$ 0.029 & 38386 $\pm$ 28  \\
\hline
0 &  &  &  & 0.8 & 1.0 &  &  0.511 $\pm$ 0.028 & 38400 $\pm$ 0  \\
\hline
0 &  &  &  & 1.0 & 1.0 &  &  0.504 $\pm$ 0.025 & 30754 $\pm$ 1272  \\
\hline
0 &  &  &  & 0.9 & 0.8 &  &  0.466 $\pm$ 0.032 & 38400 $\pm$ 0  \\
\hline
4 & 4 & 27 & 0.1 & 1.0 & 1.2 &  &  0.440 $\pm$ 0.026 & 24916 $\pm$ 954  \\
\hline
0 &  &  &  & 1.0 & 0.6 &  &  0.399 $\pm$ 0.070 & 38320 $\pm$ 73  \\
\hline
0 &  &  &  & 0.7 & 1.0 &  &  0.353 $\pm$ 0.021 & 38400 $\pm$ 0  \\
\hline
0 &  &  &  & 0.8 & 0.8 &  &  0.351 $\pm$ 0.039 & 38400 $\pm$ 0  \\
\hline
0 &  &  &  & 0.95 & 0.6 &  &  0.337 $\pm$ 0.053 & 38400 $\pm$ 0  \\
\hline
4 & 4 & 27 & 0.5 & 1.0 & 1.2 &  &  0.334 $\pm$ 0.013 & 24217 $\pm$ 1214  \\
\hline
0 &  &  &  & 0.9 & 0.6 &  &  0.284 $\pm$ 0.044 & 38400 $\pm$ 0  \\
\hline
0 &  &  &  & 1.0 & 1.2 &  &  0.269 $\pm$ 0.025 & 21906 $\pm$ 1780  \\
\hline
0 &  &  &  & 0.7 & 0.8 &  &  0.239 $\pm$ 0.019 & 38400 $\pm$ 0  \\
\hline
0 &  &  &  & 0.8 & 0.6 &  &  0.212 $\pm$ 0.011 & 38400 $\pm$ 0  \\
\hline
0 &  &  &  & 1.0 & 0.4 &  &  0.207 $\pm$ 0.029 & 38400 $\pm$ 0  \\
\hline
0 &  &  &  & 0.95 & 0.4 &  &  0.176 $\pm$ 0.013 & 38400 $\pm$ 0  \\
\hline
0 &  &  &  & 0.9 & 0.4 &  &  0.147 $\pm$ 0.013 & 38400 $\pm$ 0  \\
\hline
0 &  &  &  & 0.7 & 0.6 &  &  0.101 $\pm$ 0.028 & 38400 $\pm$ 0  \\
\hline
0 &  &  &  & 1.0 & 0.2 &  &  0.080 $\pm$ 0.020 & 38400 $\pm$ 0  \\
\hline
0 &  &  &  & 0.8 & 0.4 &  &  0.074 $\pm$ 0.027 & 38400 $\pm$ 0  \\
\hline
0 &  &  &  & 0.95 & 0.2 &  &  0.057 $\pm$ 0.018 & 38400 $\pm$ 0  \\
\hline
0 &  &  &  & 0.9 & 0.2 &  &  0.029 $\pm$ 0.015 & 38400 $\pm$ 0  \\
\hline
0 &  &  &  & 0.7 & 0.4 &  &  0.025 $\pm$ 0.016 & 38400 $\pm$ 0  \\
\hline
0 &  &  &  & 0.8 & 0.2 &  &  0.021 $\pm$ 0.014 & 38400 $\pm$ 0  \\
\hline
0 &  &  &  & 0.7 & 0.2 &  &  0.018 $\pm$ 0.011 & 38400 $\pm$ 0  \\
\hline
0 &  &  &  & 0.0 & 1.0 &  &  0.013 $\pm$ 0.000 & 38400 $\pm$ 0  \\
\bottomrule
\end{tabular}
%\end{sc}
\end{small}
\end{center}
\vspace{-4mm}
\caption{
\algoName (\Cref{alg:sampling_with_backtracking}) improves accuracy
and outperforms commonly used baselines, including
nucleus sampling top\_p,
temperature scaling (temp),
and block best-of-n (BBoN) (\Cref{sec:experiments:codellama:verifier}).
Due to space constraints, more detailed captions are in the beginning of this section.
\underline{
\textit{
\textbf{
To help readers parse these results, we included smaller tables, each analyzing a single aspect:
}
}
}
please refer to 
\Cref{table:codellama_verifier_improves_accuracy_simplified} in \Cref{sec:experiments:codellama:verifier_improves_accuracy},
\Cref{table:codellama_error_prediction_threshold} in \Cref{sec:experiments:codellama:verifier},
\Cref{table:layer27_better_than_layer31} in \Cref{sec:experiments:codellama:verifier},
and \Cref{fig:codellama_query_efficiency} in \Cref{sec:experiments:codellama:efficient}.
}
\label{table:codellama_verifier_improves_accuracy}
\end{table*}

\clearpage
\subsubsection{Visualizing the query efficiency of \algoName}
\label{sec:experiments:codellama:visualize_efficiency}

This section complements the query efficiency visualization discussed in \Cref{sec:experiments:codellama:efficient}.

\begin{figure*}[!h]
  \centering
  \begin{minipage}[b]{1.0\textwidth}  % 1.0 for arxiv
    \centering
    \includegraphics[width=1.0\textwidth]{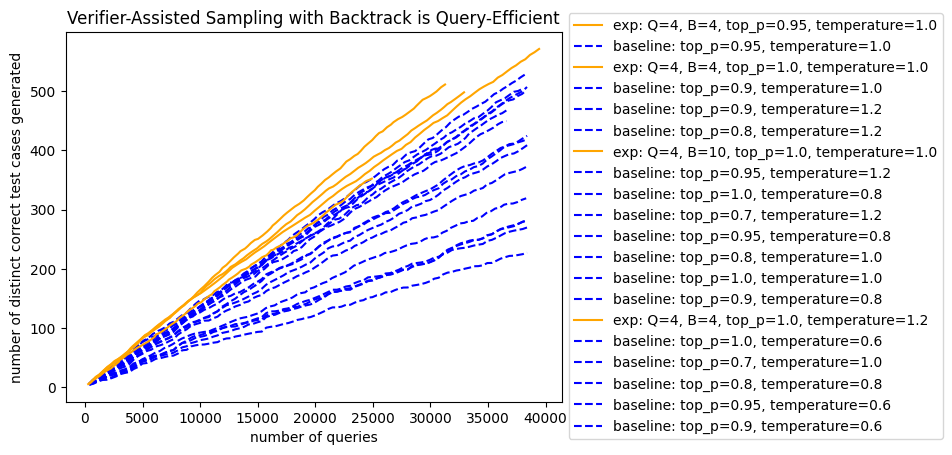}  % 0.5 for arxiv
  \end{minipage}
  % \vspace{-0.5cm}
  \caption{
  Similar to \Cref{fig:codellama_query_efficiency},
  just more zoomed-in, 
  excluding block best-of-n baselines (\Cref{sec:experiments:codellama:verifier}).
  }  
  \label{fig:codellama_query_efficiency_no_bon}
  % \vspace{-0.3cm}
\end{figure*}

\clearpage
\subsubsection{\algoName generalizes better to out-of-distribution prompts}
\label{sec:experiments:codellama:ood}

In this section, we show that 
\algoName (\Cref{alg:sampling_with_backtracking})
generalizes better to out-of-distribution prompts
than the best nucleus sampling and temperature scaling baseline in \Cref{sec:experiments:codellama:baselines}.
Unlike the synthetic Dyck grammar setting,
on real-world LLMs we do not have a precise quantitative control over how ``out-of-distribution" a prompt is for the LLM.
We therefore assume that a sufficient condition for a prompt in our setup to be out-of-distribution is that
the name of the target function denotes some meaning which is different from
the actual implemented functionality
(i.e. list \texttt{append}) 
(recall the task setup in \Cref{sec:experiments:codellama:task}).
Two examples of such out-of-distribution prompt are provided in \Cref{table:prompt_example_ood}.
We validate this assumption by observing that the accuracy indeed degrades on such ``out-of-distribution" prompts,
suggesting that the model is indeed confused by the inconsistency between the function names and the function implementations.
However, analogous to our observations on the synthetic Dyck grammar (\Cref{sec:experiments:synthetic:verifier_reduces_errors_unseen_ood}),
\algoName (\Cref{alg:sampling_with_backtracking}) again suffers much less reduction in accuracy on these ``out-of-distribution" prompts.
The detailed comparisons are reported in \Cref{table:codellama_ood}.

\begin{table}[!h]
\centering
\small
% \scalebox{0.8}{
\begin{tabular}{l}
\toprule
def f(a, b): \\
\quad \quad   return a + b \\
 \\
List 8 test cases of the above function f, one in each line: \\
assert f(6, 5) == 11 \\
assert f(3, 2) == 5 \\
assert f(5, 4) == 9 \\
assert f(1, 5) == 6 \\
assert f(5, 4) == 9 \\
assert f(3, 5) == 8 \\
assert f(5, 6) == 11 \\
assert f(2, 6) == 8 \\
 \\
def add(l, item): \\
\quad \quad   assert type(l) is list \\
\quad \quad   l.append(item) \\
\quad \quad   return l \\
 \\
List 8 test cases of the above function add, one in each line: \\
\\
\bottomrule
\\
def f(a, b): \\
\quad \quad   return a + b \\
 \\
List 8 test cases of the above function f, one in each line: \\
assert f(8, 7) == 15 \\
assert f(8, 1) == 9 \\
assert f(4, 7) == 11 \\
assert f(8, 4) == 12 \\
assert f(7, 4) == 11 \\
assert f(8, 4) == 12 \\
assert f(1, 1) == 2 \\
assert f(5, 5) == 10 \\
 \\
def exp(l, item): \\
\quad \quad   assert type(l) is list \\
\quad \quad   l.append(item) \\
\quad \quad   return l \\
 \\
List 8 test cases of the above function exp, one in each line: \\
\bottomrule
\end{tabular} 
% }
\caption{
Two example \emph{out-of-distribution} prompts for generating test cases for a simple implementation of the \texttt{append} function for Python lists.
Different from the prompts in \Cref{table:prompt_example},
here the function names denote a clear meaning
(e.g. \texttt{add} or \texttt{exp}), 
which, however, is different from what the function implements
(i.e. \texttt{append}).
}
\label{table:prompt_example_ood}
\end{table}

\begin{table*}[h]
% \vskip 0.15in
\begin{center}
\begin{small}
%\begin{sc}
\begin{tabular}{ lccccc }
\toprule
\textbf{$\backtrackQuota$} & \textbf{$\backtrackStride$} & \textbf{err threshold} & \textbf{in-distribution $\text{Acc}_\text{distinct}$ $\pm$ std err} & \textbf{OOD $\text{Acc}_\text{distinct}$ $\pm$ std err} \\
\hline
4 & 4 & 0.1 & 0.714 $\pm$ 0.011 & 0.710 $\pm$ 0.029  \\
\hline
4 & 4 & 0.5 & 0.676 $\pm$ 0.019 &  0.687 $\pm$ 0.024  \\
\hline
0 &   &    & 0.660 $\pm$ 0.042 &  0.606 $\pm$ 0.034  \\
\bottomrule
\end{tabular}
%\end{sc}
\end{small}
\end{center}
\caption{
\algoName (\Cref{alg:sampling_with_backtracking})
generalizes better to out-of-distribution prompts
than the best nucleus sampling and temperature scaling baseline in \Cref{sec:experiments:codellama:baselines},
which we identified by grid search (\Cref{table:codellama_verifier_improves_accuracy}) to be 
top\_p = 0.95,
and temperature = 1.0.
We manually pick 10 target function names according to \Cref{sec:experiments:codellama:ood}
which were unseen when training the verifier (\Cref{sec:experiments:codellama:verifier}).
When backtrack quota $\backtrackQuota = 0$,
the row denotes a baseline algorithm that does not use the verifier
(and consequently the backtrack stride $\backtrackStride$ will not matter).
The column \emph{err threshold} denotes the cutoff below which the error predictor output is interpreted as a rejection (\Cref{sec:experiments:codellama:verifier}).
When $\backtrackQuota > 0$,
the row denotes \algoName (\Cref{alg:sampling_with_backtracking}) 
with the corresponding $\backtrackQuota$ and $\backtrackStride$.
\algoName (\Cref{alg:sampling_with_backtracking})
suffered minor or no drop between in-distribution and OOD $\text{Acc}_\text{distinct}$, 
whereas the baseline suffered a drop by more than one standard error.
The experiment was repeated 5 times,
and we report the standard errors.
}
\label{table:codellama_ood}
\end{table*}

\clearpage
\subsection{Additional ablation experiments on the \algoName algorithm (\Cref{alg:sampling_with_backtracking})}
\label{sec:appendix:experiments:algo}

Besides the ablation experiments in \Cref{sec:experiments:codellama:verifier} which probe various aspects of verifier training,
in this section, we focus on one algorithmic component.

Concretely, line 10 of \algoName (\Cref{alg:sampling_with_backtracking}) 
re-generates the erased positions using argmax.
This was motivated by our results in \Cref{sec:experiments:synthetic:dyck} which suggest that \emph{out-of-distribution prefix} is a cause of generator mistakes.
As a remedy, redoing the erased positions using argmax is intended to increase the generator-predicted probability of the partially sampled generation,
which (concatenated with the prompt) will be the prefix for subsequent generation steps.
We include an ablation study verifying that this improves the accuracy,
significantly under the synthetic data setting (\Cref{table:ablation:no_argmax_dyck}),
and only slightly (without hurting diversity) under the real data setting (\Cref{table:ablation:no_argmax_codellama}).

\begin{table}[h]
% \vskip 0.15in
\begin{center}
\begin{small}
%\begin{sc}
\begin{tabular}{ lcc }
\toprule
\textbf{sampling algorithm} & \textbf{\#errors $\pm$ std err} \\
\hline
\Cref{alg:sampling_with_backtracking} & 179.4 $\pm$ 1.020 \\
\hline
ablation: no argmax & 245.8 $\pm$ 8.658 \\
\bottomrule
\end{tabular}
%\end{sc}
\end{small}
\end{center}
\caption{
Re-generating the erased positions using argmax in \algoName (\Cref{alg:sampling_with_backtracking}) reduces completion errors on unseen out-of-distribution (OOD) prefixes in Dyck grammar.
We fixed 
nucleus sampling \citep{holtzman2020the} top\_p = 0.9,
backtrack quota $\backtrackQuota = 4$,
and backtrack stride $\backtrackStride = 4$
(the best settings in \Cref{table:verifier_reduces_errors_unseen_ood}).
The row ``ablation: no argmax" refers to removing lines 9-12 in \Cref{alg:sampling_with_backtracking}.
We report the number of completion errors that occur when completing an unseen set of 10000 independently sampled out-of-distribution prompts $\dyck_{OOD}^{unseen}$.
The experiment was repeated 5 times,
and we report the standard errors.
}
\label{table:ablation:no_argmax_dyck}
\end{table}

\begin{table*}[h]
% \vskip 0.15in
\begin{center}
\begin{small}
%\begin{sc}
\begin{tabular}{ lcc }
\toprule
\textbf{sampling algorithm} & \textbf{err threshold}  & \textbf{$\text{Acc}_\text{distinct}$ $\pm$ std err} \\
\hline
\Cref{alg:sampling_with_backtracking} & 0.1 &  0.714 $\pm$ 0.011 \\
\hline
ablation: no argmax & 0.1 & 0.711 $\pm$ 0.032 \\
\hline
\Cref{alg:sampling_with_backtracking} & 0.5 & 0.676 $\pm$ 0.019 \\
\hline
ablation: no argmax & 0.5  & 0.663 $\pm$ 0.023 \\
\bottomrule
\end{tabular}
%\end{sc}
\end{small}
\end{center}
\caption{
Re-generating the erased positions using argmax in \algoName (\Cref{alg:sampling_with_backtracking}) slightly improves the accuracy-diversity tradeoff (\Cref{sec:experiments:codellama:task}) in our test case generation task.
We fixed 
nucleus sampling \citep{holtzman2020the} top\_p = 0.95,
backtrack quota $\backtrackQuota = 4$,
and backtrack stride $\backtrackStride = 4$
(the best settings in \Cref{table:codellama_verifier_improves_accuracy}).
The row ``ablation: no argmax" refers to removing lines 9-12 in \Cref{alg:sampling_with_backtracking}.
The column \emph{err threshold} denotes the cutoff below which the error predictor output is interpreted as a rejection (\Cref{sec:experiments:codellama:verifier}).
The experiment was repeated 5 times,
and we report the standard errors.
}
\label{table:ablation:no_argmax_codellama}
\end{table*}

% \newpage
% \input{sections/appendix/related_work_supp}

\end{document}